\documentclass[accepted]{uai2026}

\usepackage{url}
\usepackage{natbib}
\usepackage{amsmath,amssymb,amsthm}
\usepackage{graphicx}
\usepackage{mathtools}
\usepackage{booktabs}
\usepackage{algorithm}
\usepackage{algpseudocode}
\usepackage{placeins}
\usepackage{tikz}
\usetikzlibrary{arrows.meta,positioning,calc}
\usepackage{pgfplots}
\pgfplotsset{compat=1.18} 
\usepackage{comment}
\theoremstyle{plain}
\newtheorem{proposition}{Proposition}
\newtheorem{lemma}{Lemma}
\newtheorem{theorem}{Theorem}

\theoremstyle{definition}

\usepackage{color}
\AtBeginDocument{\hypersetup{hidelinks}}

\newcommand{\camerareadynote}[1]{
	\par\vspace{0.35em}\noindent\begingroup\color{red}\footnotesize  \par\endgroup\vspace{0.35em}}

\usepackage[normalem]{ulem}

\newtheorem{example}{Example}

\newcommand{\cS}{\mathcal{S}}
\newcommand{\cA}{\mathcal{A}}

\newcommand{\RR}{\mathbb{R}}
\newcommand{\NN}{\mathbb{N}}
\newcommand{\DM}{\Pi_{\mathrm{DM}}}
\newcommand{\CVaR}{\mathrm{CVaR}}
\title{Long-Term Sequential Decision Making Under Risk}

\author[1]{Irmaan~(Mohammad)~Mirzanejad}
\author[1]{Nadjet~Bourdache}
\author[1]{Abdel-illah~Mouaddib}
\affil[1]{Université de Caen Normandie, ENSICAEN, CNRS, Normandie Univ, GREYC UMR 6072, France}

\begin{document}
	\maketitle

	\begin{abstract}
		We study finite-horizon MDP planning under \emph{root-based} (resolute) risk objectives that apply a rank-dependent functional to the distribution of total returns. Such objectives are non-linear in the return distribution and generally break Bellman optimality, so direct optimization by scenario-tree enumeration is intractable. 
		We propose \textbf{ERQDP}, an enumeration-free and sampling-free method that solves a rank--quantile surrogate via exact DP (Dynamic Programming), evaluates candidate policies exactly by DP over return Probability Mass Functions (PMFs) on a discretized return grid (with an explicit rounding bound), and refines the surrogate in an anytime loop that reports an explicit upper--lower gap (certificate) for the target objective up to discretization budgets. 
		Across tested benchmarks, ERQDP returns certified solutions or explicit residual gaps, enables fast risk-parameter sweeps with substantial runtime gains, and supports both risk-averse and risk-seeking behaviors.
		
	\end{abstract}
  
	\section{Introduction}\label{sec:intro}

	\emph{Decision making under risk} refers
	to settings in which the possible outcomes of an action (and their probabilities)
	are known or modeled. Even in the one-shot case, the notion of an ``optimal''
	decision is subjective: two agents facing the same lottery can prefer different
	options depending on their attitude toward risk. For instance, a guaranteed gain
	of \(10\) may be preferred to a lottery paying \(50\) with probability \(0.2\) by
	a risk-averse agent, while a risk-seeking agent may prefer the lottery, and a
	risk-neutral agent would be indifferent (both have an expectation of \(10\)). This
	subjectivity motivates models that go beyond expectation by evaluating the
	\emph{entire} outcome distribution, including tail-oriented criteria such as
	Value-at-Risk/Conditional Value-at-Risk (VaR/CVaR) and more general rank-based
	(distortion) aggregators that reweight quantiles across a spectrum of outcomes
	\citep{Yaari1987,Schmeidler1989,Wang1996,BauerleJaskiewicz2024}.
	
	\emph{Sequential} decision making under risk is strictly more challenging. Actions
	are chosen repeatedly, each choice shapes future uncertainty, and the number of
	possible consequence paths (trajectories) as well as the number of possible
	strategies grows combinatorially with the horizon. Markov Decision Processes
	(MDPs) provide a standard formalism for such problems \citep{Puterman1994}. Under
	the expected-return criterion, dynamic programming (DP) applies because
	expectation is linear and satisfies the Bellman optimality principle. However, in
	domains such as finance, healthcare planning, and safety-critical operations, the
	distributional shape of the long-term return (including low-probability adverse
	outcomes) often matters at least as much as the mean, motivating risk-sensitive
	MDP solutions \citep{Kolobov2012,Ahmadi2021,BauerleJaskiewicz2024}. Most
	distributional criteria are \emph{non-linear} in the return distribution and
	therefore do not admit a straightforward Bellman decomposition.
	
	\begin{example} \label{ex:toy} Consider the following two-stage decision problem: an agent is at time $0$ at state $s_0$ and can either invest in protection (sure reward $-5$) or skip it (reward $0$). If invest, it arrives at $s_1$ and gets a certain reward of $-5$, otherwise it arrives at $s_2$ and gets a certain reward of $0$. In $s_1$ (respectively $s_2$), the agent can only select cautious (respectively aggressive) mode and gets an uncertain reward: $20$ with probability $0.99$ and $-110$ otherwise (respectively $40$ with probability $0.9$ and $-100$ otherwise). This MDP is graphically given in Figure~\ref{fig:toy_mdp} and the lotteries associated to both strategies are given in Table~\ref{tab:toy-mdp-returns}. We can see that the short-term strategy $\pi_{\text{short}}$ has a $10\%$ catastrophic outcome, whereas the long term one $\pi_{\text{long}}$ reduces catastrophe to $1\%$ which is somehow safer and could (depending on the preferences) compensate the lower rewards (compared to $\pi_{\text{short}}$).
	\end{example}
	
	\begin{table}[t]
		\centering
		\setlength{\tabcolsep}{3pt}
		\renewcommand{\arraystretch}{1.05}
		\resizebox{\columnwidth}{!}{%
			\begin{tabular}{@{}lccccc@{}}
				\toprule
				Plan & Outcomes (total return) & Prob. & $\mathbb{E}[\cdot]$ & $\CVaR_{0.1}[\cdot]$ \\
				\midrule
			    $\pi_{\text{long}}$ (invest $\to$ cautious) & $20,\ -110$ & $0.99,\ 0.01$ & $18.7$ & $7$ \\
                $\pi_{\text{short}}$ (skip $\to$ aggressive) & $40,\ -100$ & $0.9,\ 0.1$ & $26$ & $-100$ \\
				\bottomrule
		\end{tabular}}
		\caption{\small The lotteries of Example~\ref{ex:toy} and their evaluations.}
		\label{tab:toy-mdp-returns}
	\end{table}
	
	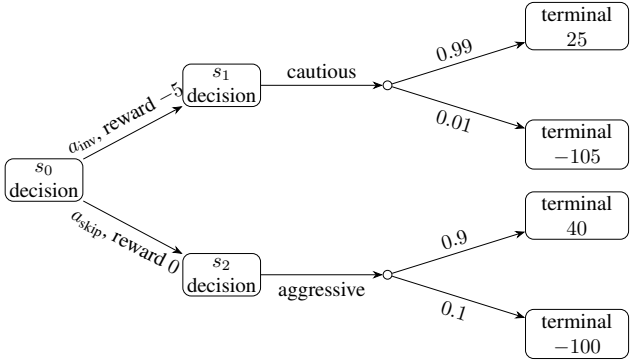
\begin{figure}[t]
		\centering
		\resizebox{\columnwidth}{!}{%
			\begin{tikzpicture}[
				>=Stealth,
				every node/.style={font=\normalsize},
				dec/.style={draw, rounded corners, inner sep=2pt, align=center},
				chance/.style={draw, circle, inner sep=1.5pt, align=center},
				term/.style={draw, rounded corners, inner sep=2pt, align=center, text width=16mm},
				]
				\node[dec]    (s0)   at (0,0)        {$s_0$\\decision};
				\node[dec]    (s1i)  at (3.0, 1.6)   {$s_1$\\decision};
				\node[dec]    (s1s)  at (3.0,-1.6)   {$s_2$\\decision};
				
				\node[chance] (ci)   at (5.8, 1.6)   {};
				\node[chance] (cs)   at (5.8,-1.6)   {};
				
				\node[term]   (iG)   at (9.0, 2.6)   {terminal\\$25$};
				\node[term]   (iB)   at (9.0, 0.6)   {terminal\\$-105$};
				\node[term]   (sG)   at (9.0,-0.6)   {terminal\\$40$};
				\node[term]   (sB)   at (9.0,-2.6)   {terminal\\$-100$};
				
				\draw[->] (s0) -- node[above, sloped] {$a_{\text{inv}}$, reward $-5$} (s1i);
				\draw[->] (s0) -- node[below, sloped] {$a_{\text{skip}}$, reward $0$} (s1s);
				
				\draw[->] (s1i) -- node[above] {cautious} (ci);
				\draw[->] (s1s) -- node[below] {aggressive} (cs);
				
				\draw[->] (ci) -- node[above, sloped] {$0.99$} (iG);
				\draw[->] (ci) -- node[below, sloped] {$0.01$} (iB);
				
				\draw[->] (cs) -- node[above, sloped] {$0.9$} (sG);
				\draw[->] (cs) -- node[below, sloped] {$0.1$} (sB);
		\end{tikzpicture}}
		\caption{\small The MDP of Example~\ref{ex:toy}.}
 		\label{fig:toy_mdp}
	\end{figure}
	
	This example illustrates why short-term screening and trajectory-level risk evaluation can lead to different policies. Furthermore, the determination of optimal strategies in both cases are quite different. While short-term screening only requires local optimal choices (with dynamic programming for example), at the root level, risk-sensitive planning selects a single policy at the initial state and evaluates it by applying a risk functional to the induced distribution of \emph{total} (long-term) returns. This couples exponentially many trajectories and generally breaks Bellman optimality. A naive optimization
		approach enumerates policies and scenario trees and then optimize over the enumerated solutions, but this quickly becomes
		intractable as the horizon grows. Existing approaches often restore tractability
		by changing the formulation (e.g., state augmentation or robust/dual
		reformulations), restricting the risk model to DP-friendly classes, or relying
		on approximations and sampling
		\citep{HowardMatheson1972,Chow2015,Lin2023b,BauerleJaskiewicz2024,AvilaPires2025}, but such reformulations or restrictions do not always apply.
	
	Moreover, while CVaR is widely used as a principled tail criterion (e.g., \citep{Rigter2022}), it is
	inherently pessimistic (it only summarizes the worst outcomes) and cannot
	express the full spectrum of behaviors that applications may require, including
	more nuanced tradeoffs across the entire distribution. This motivates flexible
	rank-based aggregators such as the Yaari model \citep{Yaari1987} also known as WOWA (Weighted Ordered Weighted Average) in the multicriteria decision making context \citep{Yager1988,Torra1997}.

	To study long-term risk in sequential decision-making problems, we focus in particular on non-linear models that can describe a wide range of preferences w.r.t risk, thereby requiring specific algorithmic solutions. In this purpose, we introduce  \textbf{ERQDP} (Exact Rank--Quantile Dynamic Programming),
	an enumeration-free framework for finite-horizon MDPs with resolute (root-level)
	distributional risk. It combines three ingredients:
	\begin{itemize}
		\item \textbf{Surrogate-to-exact risk planning:} solve a resolute risk objective via a distribution surrogate with resolution $K$, using DP to obtain a surrogate-optimal  policy.
		\item \textbf{Exact distributional evaluation:} compute the trajectory-return distribution of any policy by DP and evaluate CVaR/WOWA exactly on that distribution.
		\item \textbf{Anytime certified refinement:} increase $K$ and accept only root-improving updates, reporting an explicit optimality-gap/certificate signal.
	\end{itemize}
 
	\section{Related Work}\label{sec:related}
	
	A central distinction in risk-sensitive sequential decision making is between
	\emph{time-consistent} (nested) criteria and \emph{resolute} (root-based) criteria
	that apply a risk functional once to the total return distribution
	\citep{BauerleJaskiewicz2024}. The latter is often preferred when the goal is to
	encode a trajectory-level preference, but it is generally non-linear in the
	induced return distribution and therefore violates Bellman optimality. Root-based
	\textbf{VaR} (Value at Risk) selects a return threshold at a prescribed probability level and is
	not additive, so standard DP is inapplicable in general \citep{XiaPan2025}.
	Root-based \textbf{CVaR} (Conditional Value at Risk, a.k.a. Expected Shortfall) averages the worst tail mass and is
	coherent \citep{Artzner1999,RockafellarUryasev2000}, yet it remains time
	inconsistent under conditioning and typically regains tractability only through
	reformulation, e.g., via state augmentation or related constructions
	\citep{BauerleOtt2011,BauerleRieder2014,DingFeinberg2022} and robust/dual
	viewpoints \citep{Chow2015}. While widely used, CVaR is intrinsically tail-focused
	and can be overly pessimistic when preferences across the full distribution
	matter.
	
	A classical route is \emph{expected utility}: transform outcomes through a utility function and maximize expectation. 
	In sequential problems, applying utility \emph{locally} to one-step rewards preserves DP but changes semantics, whereas applying it \emph{globally} to the induced lottery over total returns matches resolute (root-level) preferences but typically breaks Bellman separability. 
	A canonical DP-friendly instance is \textbf{exponential/entropic utility}, which admits transformed Bellman recursions \citep{HowardMatheson1972,Jaquette1976,BauerleJaskiewicz2024}. 
	More broadly, \textbf{OCE} unifies many convex criteria (including entropic risk and CVaR) and supports both recursive (time-consistent) and ``risk outside recursion'' formulations \citep{BenTalTeboulle1986,BauerleJaskiewicz2024}; this has led to DP-style algorithms for entropic criteria such as ERM and EVaR \citep{Lin2023a,Su2025,Marthe2025}. 
	In model-free RL, coherent/convex risk criteria are often optimized via sampling-based methods \citep{Tamar2017,YuYing2023,CoacheJaimungal2024}, with complementary analyses for entropic objectives \citep{Zhang2024,MortensenTalebi2025} and dynamic time-consistent risk measures \citep{YuShen2022}; these are indispensable when the model is unknown, but they typically either adopt time-consistent semantics or provide sample-based approximations rather than certified root-level optimality for static (resolute) objectives. 
	Finally, although expected-utility models are relatively effective in terms of preference modelling, they do not, however, allow to express some (more or less) sophisticated and complex behaviours like, for example, the behaviour highlighted by the well-known Allais paradox \citep{Allais1953} (see \citep{GonzalesP2020} for more limitations of this model).

	When preferences must be controlled \emph{across} the return distribution (not
	only in an extreme tail), \textbf{distortion} and \textbf{spectral} criteria
	reweight quantiles through a probability-weighting (distortion) function
	\citep{Yaari1987,Schmeidler1989,Wang1996}. In decision theory, related
	rank-dependent utility models distort probabilities inside an expected-utility
	construction; in this paper we work with distortion \emph{risk measures} applied
	directly to return distributions, which share similar ingredients but are not the
	same object. The \textbf{WOWA} operator provides a flexible rank-based aggregation
	mechanism \citep{Yager1988,Torra1997} that can express behaviors that neither
	expected utility (too restrictive to target specific parts of the distribution)
	nor CVaR (tail-only) can capture. However, in sequential problems the dependence
	on the \emph{ordering} of trajectory returns makes optimization difficult: small
	policy changes can reshuffle ranks, breaking Bellman separability and
	complicating search over policies, 
	motivating heavy MILP/search approaches for even simpler problems: for example decision trees \citep{JeantetSpanjaard2011,JeantetPernySpanjaard2012}, or multiobjective MDPs where the reward for each objective is obtained by expectation \cite{OgryczakPernyWeng2013}.
	
	A complementary route is \textbf{distributional dynamic programming}, which
	propagates return distributions (or sufficient statistics) and then evaluates
	functionals of these distributions \citep{AvilaPires2025}. A key caution is that
	enforcing a DP decomposition for static/root objectives by discretizing ``risk
	levels'' can be inexact and yield policies that are suboptimal for the intended
	static objective \citep{Lin2023b}; related dependence on the initial state also
	appears for non-linear scalarizations in multiobjective MDPs
	\citep{OgryczakPernyWeng2013}. ERQDP avoids this pitfall by using DP only on an
	explicitly defined rank--quantile surrogate model, while evaluating the root objective \emph{exactly}
	for any explicit policy via Probability Mass Function (PMF) DP.
 
	\paragraph{Positioning of this paper.}
	The literature exposes a recurring tradeoff between \emph{tractability} and
	\emph{objective fidelity}. DP-friendly criteria preserve Bellman structure but
	either restrict the preference model or change semantics to time-consistent
	nested risk \citep{HowardMatheson1972,BenTalTeboulle1986,BauerleJaskiewicz2024}.
	Root-based quantiles/CVaR match trajectory-level evaluation but often regain
	tractability through augmentation or robust/dual reformulations
	\citep{Artzner1999,RockafellarUryasev2000,Chow2015,BauerleOtt2011,DingFeinberg2022,XiaPan2025}.
	Distortion and rank-based aggregators offer spectrum-wide control but their
	non-linearity breaks Bellman optimality and can lead to hard optimization problems
	or heavy search/MILP schemes \citep{JeantetSpanjaard2011,JeantetPernySpanjaard2012,OgryczakPernyWeng2013}. 
	
	To the best of our knowledge, ERQDP is the first DP-faithful planning framework for finite-horizon  MDPs with root-level risk consideration that satisfies, at the same time: (i) the consideration of a wide range of risk preferences (from tail-dominated to spectrum-wide); (ii) the scalability by avoiding trajectory enumeration and sampling approximations; and (iii) a theoretical guarantee of optimality.

	\section{Methodology}\label{sec:method}

	\subsection{Preliminaries: MDPs, trajectories, and rank-dependent measures}

	\paragraph{MDP, return, and resolute objective.}
	We consider a finite-horizon discounted MDP $\mathcal{M}=(\mathcal{S},\mathcal{A},P,r,\gamma,H)$ and use standard time-augmentation so a deterministic non-stationary policy is a mapping $\pi:\mathcal{S}\times\{0,\dots,H-1\}\to\mathcal{A}$.
	Given $s_0$, $\pi$ induces a trajectory $\tau=(s_0,a_0,\dots,s_H)$ with $a_t=\pi(s_t,t)$ and $s_{t+1}\sim P(\cdot\mid s_t,a_t)$, and discounted return
	\begin{equation}
		g(\tau)\;\triangleq\;\sum_{t=0}^{H-1}\gamma^t\,r(s_t,a_t,s_{t+1}).
		\label{eq:traj-return}
	\end{equation}
	Resolute planning commits to one policy at the root and maximizes a risk functional of the resulting return distribution:
	\begin{equation}
		\pi^\star\in\arg\max_{\pi\in\DM}\ \rho\!\left(G^\pi(s_0)\right).
		\label{eq:root-objective}
	\end{equation}
	\paragraph{Distortion functions and rank-dependent risk measures.}
	A broad class of law-invariant risk functionals is given by distortion (rank-dependent) measures \citep{Schmeidler1989,Yaari1987,Wang1996}.
	A distortion is a nondecreasing map $\phi:[0,1]\to[0,1]$ with $\phi(0)=0$ and $\phi(1)=1$, which reweights \emph{ranks} (cumulative probability levels) and thus encodes attitudes toward favorable vs.\ unfavorable outcomes \citep{Schmeidler1989,Yaari1987}.
	In our finite-horizon setting, returns are discrete: for $X$ with outcomes $x_1\ge\cdots\ge x_N$ (best to worst), probabilities $m_i$, and cumulative masses $C_i=\sum_{j=1}^i m_j$ ($C_0=0$), the distortion value is
	\begin{equation}
		\rho_\phi(X)=\sum_{i=1}^{N} x_i\big(\phi(C_i)-\phi(C_{i-1})\big).
		\label{eq:discrete-distortion}
	\end{equation}
	We instantiate \eqref{eq:discrete-distortion} with two criteria: CVaR (tail-only) and WOWA (spectrum-wide).
	
	\paragraph{Conditional Value at Risk (CVaR).}
	For $\alpha\in(0,1]$, $\CVaR_\alpha(X)$ is the average of the worst $\alpha$ mass of $X$:
	\begin{equation}
		\CVaR_{\alpha}(X)=\frac{1}{\alpha}\int_{1-\alpha}^{1} Q_X(p)\,dp,
		\label{eq:cvar}
	\end{equation}
	where under the best-to-worst convention
	\begin{equation}
		Q_X(p)\triangleq F_X^{-1}(1-p),\qquad p\in[0,1],
		\label{eq:btw-quantile}
	\end{equation}
	with $F_X^{-1}(p)\triangleq \inf\{x\in\RR:\ F_X(x)\ge p\}$ \citep{Dhaene2012}.
	CVaR admits classical coherent/spectral representations \citep{Artzner1999,RockafellarUryasev2000,AcerbiTasche2002,Acerbi2002} and is itself a distortion risk measure with distortion
	\begin{equation}
		\phi_\alpha(p)=
		\begin{cases}
			0, & p \le 1-\alpha,\\
			\frac{p-(1-\alpha)}{\alpha}, & p > 1-\alpha,
		\end{cases}
		\label{eq:cvar-distortion}
	\end{equation}
	so \eqref{eq:discrete-distortion} applies directly.
	
	\paragraph{Weighted Ordered Weighted Average (WOWA).}
	WOWA is a flexible rank-based aggregation operator \citep{Yager1988,Torra1997} that can be written in the distortion form \eqref{eq:discrete-distortion} for a suitable $\phi$.
	When $\phi$ is strictly increasing, $\rho_\phi$ is monotone w.r.t.\ first-order stochastic dominance; CVaR violates this because $\phi_\alpha$ is flat on $[0,1-\alpha]$ \citep{Yaari1987}.
	Different shapes recover different behaviors: $\phi(u)=u$ yields expectation, convex $\phi$ yields risk aversion, and concave $\phi$ yields risk seeking \citep{hong87,Yaari1987}.
	Accordingly, \emph{standard} (lower-tail) CVaR is inherently risk-averse, while risk seeking can be modeled via concave distortions (e.g., WOWA with $\beta<1$) or via optimistic CVaR variants such as OCVaR that optimize the upper tail \citep{AvilaPires2025}.
	\paragraph{Revisiting Example~\ref{ex:toy}}
	Table~\ref{tab:toy-risk} illustrates how tail-only and spectrum-wide criteria can induce different preferences.
	For WOWA we use the power distortion $\phi_\beta(p)=p^\beta$ and report two settings, $\beta\in\{1,2\}$ (with $\beta=1$ recovering expectation), alongside $\CVaR_\alpha$ for $\alpha\in\{0.10,0.05\}$.
	
\begin{table}[t]
	\centering
	\setlength{\tabcolsep}{3pt}
	\renewcommand{\arraystretch}{1.05}
	\begin{tabular}{@{}lcccc@{}}
		\toprule
		Plan & $\CVaR_{0.10}$ & $\CVaR_{0.05}$ & WOWA$_{\beta=1}$ & WOWA$_{\beta=2}$ \\
		\midrule
		$\pi_{\text{long}}$  & $7$    & $-6$   & $18.7$ & $17.4$ \\
		$\pi_{\text{short}}$ & $-100$ & $-100$ & $26.0$ & $13.4$ \\
		\bottomrule
	\end{tabular}
	\caption{Risk evaluations for Example~\ref{ex:toy} using total returns. WOWA uses $\phi_\beta(p)=p^\beta$ with the best-to-worst convention.}
	\label{tab:toy-risk}
\end{table}
	
	Finally, note that naively solving \eqref{eq:root-objective} by enumerating policies is infeasible due to the
	combinatorial policy space, and even evaluating a single policy by expanding its scenario tree is exponential
	in the horizon $H$. This motivates algorithms that optimize \eqref{eq:root-objective} efficiently. \citep{Strotz1955,Ruszczynski2010,BauerleOtt2011,Chow2015,BauerleJaskiewicz2024}.
	\subsection{ERQDP: Certified Rank--Quantile Dynamic Programming}

We now introduce \textbf{ERQDP}, which optimizes \eqref{eq:root-objective} through three coupled steps:
	
	(A) \emph{Surrogate-global planning:} replace the full ranked return distribution by a $K$-slice summary on the
	rank (cumulative probability) axis, and solve this induced surrogate problem exactly by backward induction (\textbf{RQDP}).
	
	(B) \emph{Exact root evaluation (no trajectory enumeration):} for any explicit policy, compute the return PMF by
	dynamic programming on a discretized return grid, and evaluate the true root objective (up to a controlled rounding error).
	
	(C) \emph{Anytime refinement with root-safe updates:} increase $K$ to refine the surrogate resolution, but accept a new
	policy only if its \emph{exactly evaluated} root objective improves, ensuring monotone improvement under resolute semantics.
 	
	\paragraph{(A) Rank--Quantile surrogate and Dynamic Programming (RQDP).}
	
	We define a surrogate approximation of \eqref{eq:discrete-distortion} to make the optimization feasible.
	The main idea is that, instead of enumerating and sorting the full return distribution of a policy,
	we approximate it by grouping outcomes according to their \emph{rank} i.e., their position along the cumulative probability axis into $K$ slices.
	
	\textbf{Rank axis, grid, and slices}
	Recall that rank-dependent criteria weight outcomes according to their position in the sorted list.
	Under the best-to-worst convention \eqref{eq:btw-quantile}, the variable $p\in[0,1]$ plays the role of a
	\emph{rank (cumulative probability) level}: $p=0$ indexes the best end of the distribution (top quantiles) and increasing $p$ moves toward worse quantiles, with $p=1$ indexing the worst end.
	It is not a probability of a single outcome; it is the cumulative-mass parameter used to traverse the sorted return distribution.

	For a fixed resolution $K$, we discretize this rank axis into a grid
	\[
	0=p_0<p_1<\cdots<p_K=1,
	\]
	
	typically the uniform grid $p_k=k/K$, in which case $q_k = 1/K$ for all $k$.
	Each interval $(p_{k-1},p_k]$ is a \emph{slice} with probability mass (quota)
	$q_k\triangleq p_k-p_{k-1}$, so $\sum_{k=1}^K q_k=1$.
	The integer $K$ controls the resolution: larger $K$ yields finer rank slices.
	
	\textbf{Piecewise-linear distortion on the grid.}
	Let $\phi_K$ be the piecewise-linear interpolant of $\phi$ on this grid and define the slice weights
	\[
	\Delta\phi_k \triangleq \phi_K(p_k)-\phi_K(p_{k-1}), \qquad k=1,\dots,K.
	\]
	
	Then the rank-dependent objective of \eqref{eq:discrete-distortion} reduces to the weighted sum 
	\begin{equation}
		\rho_{\phi_K}(X)=\sum_{k=1}^K \Delta\phi_k\,\bar Q_X(k)
		=\langle \Delta\phi^{(K)},\bar Q_X\rangle,
		\label{eq:phiK-slice}
	\end{equation}

	With $\Delta\phi^{(K)}\triangleq(\Delta\phi_1,\dots,\Delta\phi_K)\in\RR^K$, we summarize a return lottery $X$
	by a $K$-vector $\bar Q_X\in\RR^K$ that averages the best-to-worst quantile function $Q_X$ over $K$ rank slices.
	Let $0=p_0<p_1<\cdots<p_K=1$ be a rank grid and $q_k\triangleq p_k-p_{k-1}$ the mass of slice $k$.
	For each $k\in\{1,\dots,K\}$ we define the \emph{slice-average quantile}
	\begin{equation}
		\bar Q_X(k)\triangleq \frac{1}{q_k}\int_{p_{k-1}}^{p_k} Q_X(p)\,dp .
		\label{eq:slice-avg-quantile}
	\end{equation}
	In words, $\bar Q_X(k)$ is the average return level among outcomes whose \emph{rank} (cumulative probability position
	in the best-to-worst ordering) lies in the interval $(p_{k-1},p_k]$.
	
	Since $X$ is discrete in our setting, $Q_X$ is a step function. Writing the distinct outcomes of $X$ as
	$x_1\ge \cdots \ge x_N$ with probabilities $m_i$ and cumulative masses $C_i=\sum_{j=1}^i m_j$ (with $C_0=0$),
	the integral reduces to the finite overlap sum
	\begin{equation}
		\begin{aligned}
			\bar Q_X(k)
			&=\frac{1}{q_k}\sum_{i=1}^{N} x_i\,
			\lambda\!\big((C_{i-1},C_i]\cap(p_{k-1},p_k]\big),\\
			&\hspace{6em} k=1,\dots,K .
		\end{aligned}
		\label{eq:slice-avg-quantile-discrete}
	\end{equation}
	where $\lambda(\cdot)$ denotes interval length on $[0,1]$ (i.e., how much probability mass of outcome $x_i$
	falls inside slice $k$). This is exactly what the quota-filling step in our RQDP backup computes.
	Therefore, optimizing $\rho_{\phi_K}$ can be done over the $K$-dimensional summary $\bar Q_X$ rather than the full
	(return) distribution. An explicit $K{=}2$ quota-filling computation is given in Appendix~\ref{app:toy_slice_example}.

	\paragraph{Surrogate planning problem}
	In \eqref{eq:root-objective} we therefore replace $\rho$ by the surrogate formulation \eqref{eq:phiK-slice}:
	\begin{equation}
		\pi^\star_K \in \arg\max_{\pi\in\DM}\ \rho_{\phi_K}\!\left(G^\pi(s_0)\right).
		\label{eq:root-objective-surrogate}
	\end{equation}
	
	Given the slice-average vector $\bar Q_X\in\RR^K$, the surrogate evaluation $\rho_{\phi_K}(X)=\langle \Delta\phi^{(K)},\bar Q_X\rangle$ is a \emph{linear functional of this $K$-vector}.
	Dynamic programming applies because we define an induced $K$-slice \emph{surrogate return model} whose “state” is exactly this vector summary and whose backup computes the correct parent summary from successor summaries.

	In Algorithm 1, we associate to each state $s$ (and remaining horizon $h$) a $K$-vector $V_h(s)\in\RR^K$, interpreted as the slice-average summary \eqref{eq:slice-avg-quantile} of the
	$h$-step discounted return distribution from $s$ under the surrogate semantics.

	\paragraph{Backup} For each state $s$, and given successor summaries $V_{h-1}(s')\in\RR^K$, RQDP computes for each action $a$ a candidate parent vector $\bar Q_h(s,a)$ by:
	\begin{enumerate}
		\item \textbf{atomize} successor slices into a list of value--mass pairs. \emph{Atomization} means converting the successor $K$-slice summaries into a finite list of value--mass pairs (“atoms”) whose mixture defines the one-step-ahead surrogate return distribution. Let $t \triangleq H-h$ denote the stage index from the root (so the one-step reward is discounted by $\gamma^t$). Then $\{(x_{s',k},m_{s',k})\}$ that defines a discrete distribution over returns, where 
		\begin{align}
			x_{s',k}&=\gamma^{t}\,r(s,a,s')+V_{h-1}(s')(k), \label{eq:atoms_x}
			\qquad \\
			m_{s',k}&=\Pr(s'\mid s,a)\,q_k.
			\label{eq:atoms_m}
		\end{align}
		Each atom $(x_{s',k},m_{s',k})$ represents the entire rank slice $k$ of the successor distribution collapsed to its slice-average value. Importantly, parent slices are \emph{not} computed “diagonally” in $k$: after all atoms are pooled across all successors and all $k$, we re-sort them globally and quota-fill the parent slices.
		Thus, $\bar Q_h(s,a)(k)$ can depend on many child indices $k'$ through this re-ranking step.
		\item \textbf{sort and quota-fill.}
		Sort all atoms by value $x_{s',k}$ from best to worst.
		Initialize remaining slice capacities $\mathrm{cap}_k\gets q_k$ for $k=1,\dots,K$ and initialize slice accumulators
		$\mathrm{num}_k\gets 0$.
		Scan the sorted atoms in order; for the current atom with remaining mass $m$ and value $x$, repeatedly:
		allocate $\delta\gets \min\{m,\mathrm{cap}_k\}$ to the current slice $k$,
		update $\mathrm{num}_k\gets \mathrm{num}_k+\delta\,x$, $\mathrm{cap}_k\gets \mathrm{cap}_k-\delta$, and $m\gets m-\delta$;
		when $\mathrm{cap}_k=0$, increment $k$.
		At the end, define $\bar Q_h(s,a)(k)\triangleq \mathrm{num}_k/q_k$.
	\end{enumerate}

	\begin{algorithm}[t]
		\caption{RQDP: Rank--Quantile DP (fixed $K$)}
		\label{alg:rqdp}
		\begin{algorithmic}[1]
			\Require distortion $\phi$ (evaluated on grid points), horizon $H$, rank resolution $K$
			\State Define rank grid $0=p_0<\cdots<p_K=1$; quotas $q_k\gets p_k-p_{k-1}$
			\State Define slice weights $\Delta\phi_k \gets \phi(p_k)-\phi(p_{k-1})$ for $k=1,\dots,K$
			\ForAll{$s\in\cS$} \State $V_0(s)\gets \mathbf{0}\in\RR^K$ \EndFor
			
			\For{$h=1$ to $H$} \Comment{$h$ = remaining steps}
			\State $t\gets H-h$ \Comment{stage index from the root (for discounting)}
			\ForAll{$s\in\cS$}
			\If{terminal$(s)$} \State $V_h(s)\gets \mathbf{0}$; \textbf{continue} \EndIf
			\ForAll{$a\in\cA(s)$}
			\State $\bar Q_h(s,a)\gets \textsc{Backup}(s,a,t,V_{h-1}(\cdot),q)$
			\EndFor
			\State $\pi_h(s)\in\arg\max_{a\in\cA(s)}\ \langle \Delta\phi,\bar Q_h(s,a)\rangle$
			\State $V_h(s)\gets \bar Q_h(s,\pi_h(s))$
			\EndFor
			\EndFor
			\State \Return $\pi=\{\pi_h\}_{h=1}^H$ (and optionally $V_H(\cdot)$)
		\end{algorithmic}
	\end{algorithm}

	The presented algorithm is optimal for the surrogate model:
 	We justify that Algorithm~\ref{alg:rqdp} solves the induced $K$-slice surrogate problem
	\eqref{eq:root-objective-surrogate}. The only technical point is that the Backup step
	(atomize, sort, quota-fill) computes the slice-average vector $\bar Q$ of the discrete distribution
	represented by the atoms; this is formalized in Lemma~\ref{lem:quota}. The surrogate optimality then
	follows by standard finite-horizon backward induction.
	
	\begin{lemma}[Quota filling computes slice-average quantiles]
		\label{lem:quota}
		Fix a rank grid \(0=p_0<p_1<\cdots<p_K=1\), with slice masses
		\(q_k=p_k-p_{k-1}\). For any discrete return distribution with sorted outcomes
		\(x_1\ge\cdots\ge x_N\), probabilities \(m_i\), and cumulative masses
		\(C_i=\sum_{j=1}^i m_j\), the sorting and quota-filling procedure returns exactly
		the slice-average vector \(\bar Q_X\) defined in
		\eqref{eq:slice-avg-quantile}--\eqref{eq:slice-avg-quantile-discrete}.
	\end{lemma}
	
	\begin{proof}[Proof]
		The best-to-worst quantile function is constant and equal to \(x_i\) on each
		interval \((C_{i-1},C_i]\). Quota filling scans these same intervals in sorted
		order and allocates to slice \(k\) exactly the overlap
		\((C_{i-1},C_i]\cap(p_{k-1},p_k]\). Thus each slice accumulator is precisely
		the numerator of \eqref{eq:slice-avg-quantile-discrete}. Full details are given
		in Appendix~\ref{app:proofs}.
	\end{proof}

	\begin{proposition}[Surrogate optimality of RQDP]
		\label{prop:rqdp-opt}
		For fixed \(K\) and the induced piecewise-linear distortion \(\phi_K\),
		Algorithm~\ref{alg:rqdp} returns a deterministic Markov policy that maximizes
		the surrogate root objective in \eqref{eq:root-objective-surrogate}.
	\end{proposition}
	
	\begin{proof}[Proof]
		For each state, action, and remaining horizon, Lemma~\ref{lem:quota} shows that
		the backup computes the exact \(K\)-slice summary of the induced surrogate
		return lottery. Since the surrogate objective is linear in this summary,
		\(\langle\Delta\phi^{(K)},\bar Q\rangle\), selecting the action with maximal
		score is Bellman-optimal for the induced surrogate model. Finite-horizon backward
		induction then proves surrogate optimality. Full details are in
		Appendix~\ref{app:proofs}.
	\end{proof}
	
	\paragraph{(B) Exact enumeration-free policy evaluation via PMF dynamic programming} We nee to evaluate a candidate policy $\pi$ (obtained with Algorithm \ref{alg:rqdp}) under the resolute measure \eqref{eq:discrete-distortion}. This deterministic method computes the induced return distribution of $G^\pi(s_0)$ \emph{without} trajectory enumeration, Monte Carlo rollouts, or stochastic sampling, and then evaluates it using $\rho$.

	We evaluate a fixed policy $\pi$ by dynamic programming over return PMFs on a discretized grid of step $1/c$ (returns rounded to $\{j/c:\,j\in\mathbb{Z}\}$), restricted to the bounded interval $[G_{\min},G_{\max}]$.
 	Using global bounds $G_{\min}\le G^\pi(s_0)\le G_{\max}$ (e.g., from reward bounds), we restrict to indices
	$j\in\{ \lceil cG_{\min}\rceil,\dots,\lfloor cG_{\max}\rfloor\}$.
	Here $c\in\NN$ controls resolution: larger $c$ gives a finer grid resolution(step size $1/c$).

	For a state $s$ and a remaining horizon $h$, let $\mathrm{PMF}_h^\pi(s)$ denote
	the distribution of the remaining $h$-step discounted return from $s$ when following 
 	the policy $\pi$.

	Here $\delta_0$ denotes the Dirac mass at return $0$ (probability $1$ at grid index $0$).
	With base $\mathrm{PMF}_0^\pi(s)=\delta_0$, we propagate PMFs backward by mixing shifted successor distributions:
	\begin{equation}
		\begin{split}
			\mathrm{PMF}_h^\pi(s)
			&=
			\sum_{s'} P(s'\mid s,\pi_h(s)) \;
			\mathrm{Shift}\!\Bigl(
			\mathrm{PMF}_{h-1}^\pi(s'), \\
			&\qquad
			\mathrm{round}\!\bigl(
			c\,\gamma^{t}\,r(s,\pi_h(s),s')
			\bigr)
			\Bigr)
		\end{split}
		\label{eq:pmf-recursion}
	\end{equation}
	
	where $t=H-h$ is the stage index from the root. Here $\mathrm{Shift}(\cdot,k)$
	shifts a PMF by $k$ integer grid points, and $\mathrm{round}(\cdot)$ rounds to
	the nearest integer.

	\begin{proposition}[Enumeration-free PMF evaluation and rounding guarantees]
		\label{prop:pmf-eval}
		For a fixed policy $\pi$ and discretization parameter $c$, the recursion \eqref{eq:pmf-recursion} computes the distribution of the rounded return $\tilde G^\pi(s_0)$ on the grid $\{j/c\}$.
		If all discounted rewards $\gamma^t r(s_t,a_t,s_{t+1})$ lie on the grid (exact alignment), then $\tilde G^\pi(s_0)=G^\pi(s_0)$.
	 
		Otherwise, for every trajectory,
		\[
		\big|G^\pi(s_0)-\tilde{G}^\pi(s_0)\big|\ \le\ \frac{H}{2c}
		\quad\text{a.s.}
		\]
		and for any distortion $\phi$,\[
		\big|\rho_\phi(G^\pi(s_0))-\rho_\phi(\tilde{G}^\pi(s_0))\big|
		\ \le\ \frac{H}{2c}.
		\]
	\end{proposition}
	
	\begin{proof}[Proof]
		For a fixed policy, the return distribution satisfies a standard law-of-total
		probability recursion over the next state. Equation~\eqref{eq:pmf-recursion}
		implements this recursion by shifting each successor PMF by the rounded
		one-step discounted reward and mixing by transition probabilities. Each rounded
		reward differs from the true discounted reward by at most \(1/(2c)\); summing
		over \(H\) stages gives the trajectory-wise bound \(H/(2c)\). Since distortion
		values are monotone weighted integrals of quantiles, the same uniform error
		bounds the risk value. Full details are in Appendix~\ref{app:proofs}.
	\end{proof}

	\paragraph{(C) Anytime refinement and root-safe improvements.}
	ERQDP is an anytime loop over increasing rank resolutions $K$ (e.g., doubling).
	For each $K$ update:
 
	after running (A) and (B):
	(1) perform root-safe improvement sweeps that consider one-step
	deviations (take an action $a$ outside the considered policy and then follow the incumbent policy) Meaning, at a selected state--stage pair $(s,h)$, we try each $a\in\cA(s)$ as a one-step deviation:
	take $a$ at $(s,h)$, then follow the incumbent policy for the remaining step and accept a
	changed policy only if it strictly improves the \emph{root} objective, avoiding any reliance on time-consistency \citep{Strotz1955,Ruszczynski2010};
	(2) prune actions using surrogate envelopes to reduce computation in
	large action spaces.
	
	\noindent\textbf{ERQDP (Algorithm~\ref{alg:erqdp-outline})}
	For any explicit policy $\pi$, let
	\[
	J(\pi)\ \triangleq\ \rho_\phi\!\left(\tilde G^\pi(s_0)\right)
	\]
	denote the root objective evaluated on the return grid (Proposition~\ref{prop:pmf-eval}); in particular $J_K \triangleq J(\pi_K)$.
	Let $\phi_K$ be the $K$-slice surrogate distortion and define the surrogate-optimal value
	\[
	\widehat J_K^\star\ \triangleq\ \max_{\pi\in\DM}\ \rho_{\phi_K}\!\left(G^\pi(s_0)\right),
	\]
	which is returned by RQDP (Proposition~\ref{prop:rqdp-opt}).
	We maintain a lower bound $LB \triangleq \max\{J(\pi)\text{ over evaluated policies }\pi\}$ and an upper bound
	\[
	\begin{aligned}
		UB &\triangleq \widehat J_K^\star + \mathcal E(K,c),\\
		\mathcal E(K,c) &\triangleq (G_{\max}-G_{\min})\,\delta_K + \frac{H}{2c}.
	\end{aligned}
	\]
	where $\delta_K=\|\phi-\phi_K\|_\infty$. Then $UB$ upper-bounds the optimal root value for $\rho_\phi$ up to the discretization/rounding budgets, and the stopping test $UB-LB\le\varepsilon$ yields the reported certificate gap.
	The entire ERQDP algorithm is summarized in Algorithm 2.
	
	\begin{algorithm}[t]
		\caption{ERQDP  }
		\label{alg:erqdp-outline}
		\begin{algorithmic}[1]
			\Require distortion $\phi$, horizon $H$, start state $s_0$, initial rank resolution $K_0$, max resolution $K_{\max}$, tolerance $\varepsilon$, return-grid scale $c$
			\State $K \gets K_0$
			\State Initialize an incumbent policy $\pi_{\text{best}}$
			\State $LB \gets$ exact root value of $\pi_{\text{best}}$ via PMF-DP (grid scale $c$)
			\Repeat
			\State \textbf{(A)} Build $\phi_K$ (the $K$-slice surrogate of $\phi$) and run RQDP to get a candidate policy $\pi_K$
			\State $\widehat J_K^\star \gets$ surrogate root score of $\pi_K$ (ret. by RQDP)
			
			\State \textbf{(B)} $J_K \gets$ exact root value of $\pi_K$ (original MDP) via PMF-DP (grid scale $c$)
			\If{$J_K > LB$} $\pi_{\text{best}} \gets \pi_K$; $LB \gets J_K$ \EndIf

			\State \textbf{(C)} Try local policy changes around $\pi_{\text{best}}$; accept only if the exact root value improves; update $LB$
			
			\Statex \textbf{Stopping / refine}
			\State $UB \gets \widehat J_K^\star + \mathcal E(K,c)$
			\If{$UB - LB \le \varepsilon$ \textbf{or} $K = K_{\max}$} \State \textbf{break} \EndIf
			\State $K \gets \min\{2K,\ K_{\max}\}$
			\Until{false}
			\State \Return $\pi_{\text{best}}$, $LB$, $UB$ \Comment{or certificate $UB-LB$}
		\end{algorithmic}
	\end{algorithm}

	In this algorithm, there are two independent sources of approximation:
	(i) rank discretization through $\phi_K$ (and the $K$-slice surrogate), and
	(ii) return-grid rounding in PMF evaluation (Proposition~\ref{prop:pmf-eval}). The following result shows that the distance between the approximated evaluation and the exact one is bounded. A standard stability bound for distortion integrals controls sensitivity to the
	distortion function. Let $\delta_K=\|\phi-\phi_K\|_\infty$.
	
	On the uniform grid $p_k=k/K$, $\delta_K$ decreases as $O(1/K)$ for Lipschitz $\phi$ (and more generally $\delta_K\le L\max_k(p_k-p_{k-1})$).
	
	\begin{theorem}[Distortion stability under $\|\cdot\|_\infty$]
		\label{thm:lipschitz}
		If $X\in[m,M]$ almost surely and $\phi,\psi$ are distortions, then
		\begin{equation}
			\big|\rho_\phi(X)-\rho_\psi(X)\big|\le (M-m)\,\|\phi-\psi\|_\infty.
			\label{eq:lipschitz}
		\end{equation}
	\end{theorem}
	\begin{proof}[Proof]
		Write the distortion value as the Stieltjes integral
		\(\rho_\phi(X)=\int_0^1 Q_X(p)\,d\phi(p)\). Then the difference between two
		distortions is
		\[
		\rho_\phi(X)-\rho_\psi(X)=\int_0^1 Q_X(p)\,d(\phi-\psi)(p).
		\]
		Since \(Q_X\) is monotone and ranges over an interval of length at most
		\(M-m\), this integral is bounded by
		\((M-m)\|\phi-\psi\|_\infty\). Full details are in
		Appendix~\ref{app:proofs}.
	\end{proof}
 
	In practice, we compute $\delta_K=\|\phi-\phi_K\|_\infty$ either by a provable per-interval maximization (for common distortions such as the power family and the CVaR kink), or by deterministic grid evaluation on $[0,1]$ with a conservative margin. Combined with global return bounds $[G_{\min},G_{\max}]$ (e.g., if one-step rewards satisfy $r\in[r_{\min},r_{\max}]$ then $G_{\min}=\sum_{t=0}^{H-1}\gamma^t r_{\min}$ and $G_{\max}=\sum_{t=0}^{H-1}\gamma^t r_{\max}$) and the PMF rounding budget of Proposition~\ref{prop:pmf-eval}, this yields a gap diagnostic. We use the term \emph{$\varepsilon$-certificate} only in a \textbf{surrogate-scoped} sense: it becomes formal when $\delta_K$ is provably upper-bounded and when the objective is interpreted on the induced $K$-slice surrogate model; for the original objective it remains an \emph{a posteriori} diagnostic absent additional bounds on surrogate-to-true compression.

	\noindent\textbf{Distance to optimality and certificate interpretation.}
	Let \(\pi_K^\star\) be the surrogate-optimal policy returned by RQDP for
	\(\phi_K\), and let \(\pi^\star\) be an optimal policy for the original
	distortion \(\phi\). By Theorem~\ref{thm:lipschitz} and
	Proposition~\ref{prop:pmf-eval}, for any evaluated policy \(\pi\),
	\[
	\big|\rho_\phi(\pi)-\rho_{\phi_K}(\pi)\big|
	\le \mathcal E(K,c).
	\]
	Since \(\pi_K^\star\) maximizes the surrogate objective exactly, the optimality
	loss of the incumbent is controlled by
	\[
	\rho_\phi(\pi^\star)-\rho_\phi(\pi_{\mathrm{best}})
	\le
	\Bigl(\rho_{\phi_K}(\pi_K^\star)-\rho_\phi(\pi_{\mathrm{best}})\Bigr)
	+\mathcal E(K,c).
	\]
	This is the reason ERQDP reports the deterministic gap \(\mathrm{UB}-\mathrm{LB}\). If the gap is
	below the target tolerance, the incumbent is certified under the selected
	discretization budgets; otherwise, ERQDP still returns the policy together with
	an explicit residual optimality-gap bound. See Appendix~\ref{app:proofs}.

	\section{Experiments}

	\paragraph{Common protocol.}
	All experiments use finite-horizon tabular MDPs and evaluate a policy by the distribution of total return from a specified initial state. ERQDP planning, policy evaluation, and certificate computation are deterministic: they use rank--quantile dynamic programming and PMF-DP, not Monte Carlo rollouts or trajectory sampling. Throughout, ``Opt. gap'' denotes the final deterministic upper--lower optimality-gap bound returned by ERQDP under the stated discretization budgets. Values below the run tolerance correspond to certified rows; larger values remain explicit residual bounds. Runtime is end-to-end planning time. We use three benchmark families as follows.
	
	\subsection{Betting Game (BG)}

	\label{sec:bg}
	
	BG is a ten-stage wealth-allocation problem from \citet{Rigter2022}, with horizon $H=10$ and discount factor $\gamma=1$. A state $(t,m)$ consists on one time step $t\in\{0,\ldots,10\}$ and a current wealth $m\in\{0,\ldots,100\}$; the initial state is $(0,5)$. At each nonterminal stage the agent chooses a bet $b\in\{0,\ldots,5\}$, clipped by available wealth. The transition increases wealth by $b$ with probability $0.7$, by $10b$ with probability $0.05$, and decreases it by $b$ with probability $0.25$, with wealth clipped to $[0,100]$. The terminal cost is the shortfall $C:=100-m_T$; our implementation uses the equivalent final reward $m_T-100$, so minimizing cost is equivalent to maximizing return. We report mean cost and tail cost $\CVaR_\alpha(C)$. ERQDP tracks an upper bound on the optimal risk value and reports the final policy gap. If this gap is below the strict run tolerance, the row is certified; otherwise the same number is an explicit residual bound. Rigter et al.\ report CVaR-based planning, so WOWA objective values are not directly comparable to their optimized criterion. We therefore compare policies under common cost metrics, regardless of the objective used to obtain them. Table~\ref{tab:bg_wowa_vs_rigter_cvar02} summarizes representative WOWA settings spanning risk-seeking ($\beta<1$), risk-neutral ($\beta=1$), and risk-averse ($\beta>1$). Additional BG results are in Appendix~\ref{app:bg_extended}.
	
	\begin{table}[t]
		\centering
		\setlength{\tabcolsep}{3pt}
		\caption{BG (cost; lower is better). Policies tuned for different objectives are evaluated under common cost metrics. WOWA objective values are \emph{not} comparable to CVaR objectives,
			so comparison is via mean / CVaR / VaR.}
		\label{tab:bg_wowa_vs_rigter_cvar02}
		\resizebox{\columnwidth}{!}{%
			\begin{tabular}{l l c r r r l}
				\toprule
				Source & Measure & Tuned for & Mean & $\mathrm{CVaR}_{0.2}$ & $\mathrm{VaR}_{0.2}$  & Opt. gap \\
				\midrule
				Rigter et al. & EV       & --            & 58.58 & 97.41 & -- &   -- \\
				Rigter et al. & CVaR-WC & $\alpha=0.2$   & 82.86 & 91.84 & -- &   -- \\
				Rigter et al. & CVaR-EV & $\alpha=0.2$   & 75.33 & 91.79 & 87 &  -- \\
				\addlinespace
				ERQDP & CVaR & $\alpha=0.2$ & 95 & 95 & 95   & 1.4e-14 \\
				\addlinespace
				ERQDP & WOWA & $\beta=0.5$  & 61.73 & 99.71 & 94  & 9.4e-03 \\
				ERQDP & WOWA & $\beta=1$    & 61.65 & 97.91 & 87 &   $\mathbf{0}$ \\
				ERQDP & WOWA & $\beta=2$    & 62.03 & 96.59 & 88 &   $\mathbf{0}$ \\
				\bottomrule
			\end{tabular}%
		}
	\end{table}

	\subsection{Certified planning on GridWorld (GW) return-risk metrics}

	\label{sec:gridworld_compare}
	
	GW is a stochastic navigation benchmark on a $5\times5$ grid with row-major cell indices $0,\ldots,24$, start cell $0$, horizon $H=20$, and discount $\gamma=0.9$. The agent chooses one of four compass actions. With probability $1-p_{\mathrm{slip}}$ the intended action is executed, and with probability $p_{\mathrm{slip}}$ the executed action is redrawn uniformly from the four compass moves. Entering a goal cell terminates the episode with reward $+50$; entering a penalty cell gives the negative reward listed thereafter. We report four matched instances, denoted GW1 to GW4: GW1 (respectively GW2, GW3 and GW4) use goal cell $21$ (respectively $21$, $21$ and $24$) with penalties on cells $\{10,12\}$ (respectively $\{10,12\}$, $\{12\}$, and $\{16\}$) corresponding to reward $-15$ (respectively $-15$, $-5$ and $-30$), and slip probability of $0.1$ (respectively $0.3$, $0.1$, and $0.1$). Appendix~\ref{app:gw_layouts} gives the exact layouts.
	
	We compare ERQDP with the reported \citet{AvilaPires2025} results on shared return-risk metrics. ERQDP is a model-based planner/certifier, so the reported speedup compares baseline training time with ERQDP solve time rather than identical runtime modes. Table~\ref{tab:gridworld_cvar_compare} reports the median ERQDP optimality gap, solve time, speedup, and median differences in $\CVaR_\alpha$ and mean return, with positive differences favoring ERQDP. Table~\ref{tab:gridworld_cvar_compare} also shows the main pattern. For $\alpha\ge0.10$, ERQDP has zero or small median optimality gaps, large training-time/solve-time speedups, and positive median tail-return gains while preserving or improving mean return. At extreme tails ($\alpha\le0.05$), the problem becomes more demanding; the reported gap quantifies the remaining bound at the tested resolution. 
	Beyond lower-tail CVaR, the same ERQDP pipeline supports WOWA distortions for risk-seeking, risk-neutral, and risk-averse behavior. Appendix~\ref{app:ocvar_vs_wowa} compares this spectrum-wide modeling option with OCVaR, an optimistic upper-tail CVaR variant.
	
	\begin{table}[t]
		\centering
		\setlength{\tabcolsep}{2.5pt}
		\renewcommand{\arraystretch}{1.05}
		\caption{GW CVaR sweep. Positive deltas mean ERQDP achieves \emph{higher} tail-return and/or mean return. Speedup is \citet{AvilaPires2025} reported training time divided by ERQDP solve time, using medians across instances.}
		\label{tab:gridworld_cvar_compare}
		
		\resizebox{\columnwidth}{!}{%
			\begin{tabular}{c r r r r r}
				\toprule
				$\alpha$ &
				Med. opt. gap $\downarrow$ &
				Solve(s) $\downarrow$ &
				Speedup($\times$) $\uparrow$ &
				$\Delta \mathrm{CVaR}_\alpha$ $\uparrow$ &
				$\Delta \mathbb{E}[G]$ $\uparrow$ \\
				\midrule
				0.01 & 8.46e-02 & 724.0  & 6.44   & -0.58 & -10.08 \\
				0.05 & 1.98e-02 & 1219.7 & 1.04   & -4.42 & -10.96 \\
				0.10 & 0.0e+00 & 212.0  & 78.49  & +7.16 & +6.11 \\
				0.25 & 4.80e-03 & 7.49   & 259.47 & +5.02 & +2.82 \\
				0.50 & 3.41e-03 & 1.40   & 928.18 & +9.42 & +9.47 \\
				1.00 & 4.34e-03 & 1.70   & 772.77 & +9.73 & +9.33 \\
				\bottomrule
			\end{tabular}%
		}
		
	\end{table}

	\subsection{Inventory Control (IC)}
 
	\label{sec:inventory_control}
	
	Inventory Control (IC) is reported in detail in Appendix~\ref{app:ic_rigter}, because the main comparison with \citet{Rigter2022} is summarized by a single cost table there. IC is a ten-period replenishment problem with capacity $N=20$, discount $\gamma=1$, state $(t,n,d_{\mathrm{prev}})$, initial inventory $10$, and initial previous demand $10$. At each period the agent chooses an order quantity $a\in\{0,\ldots,20-n\}$, equivalently implemented with $21$ action indices and capacity clipping. Demand follows a clipped random walk with increments in $\{-5,\ldots,5\}$, and the policy trades purchasing and holding costs against revenue and rare high-cost outcomes. On this benchmark, ERQDP-CVaR at $\alpha=0.02$ attains CVaR$_{0.02}$ cost $373.26$ (about $3.4\%$ lower than the best Rigter et al.\ CVaR baseline), while additionally outputting an explicit a posteriori optimality-gap bound.

	\subsection{Parameter Sensitivity and Scalability}

	\label{sec:param_scalability}
	
	Our algorithm contains two discretization parameters that play different roles. The rank resolution $K$ controls the number of slices used by RQDP on the cumulative-probability axis. Increasing $K$ tightens the approximation of the distortion function and usually decreases the certificate gap, but it increases the cost of the RQDP backup because each successor contributes $K$ atoms. The return-grid scale $c$ controls deterministic policy evaluation: returns are rounded to multiples of $1/c$, and Proposition~\ref{prop:pmf-eval} gives the worst-case rounding contribution $H/(2c)$. Thus, $K$ primarily affects the surrogate optimality envelope, whereas $c$ affects PMF evaluation accuracy and memory. When reporting normalized gaps, we use the return
	$\mathrm{span}=G_{\max}-G_{\min}$, where $G_{\min}$ and $G_{\max}$ are the global cumulative-return bounds.

	For fixed $K$ and $c$, ERQDP remains polynomial in the tabular model size and horizon and avoids exponential trajectory or policy enumeration. If $B$ is the maximum number of successors per state--action pair, each RQDP backup sorts $\mathrm{BK}$ atoms and costs $O(\mathrm{BK}\log(\mathrm{BK}))$ time. A full finite-horizon RQDP sweep costs $O(H|\mathcal S||\mathcal A|\mathrm{BK}\log(\mathrm{BK}))$ time, with $O(|\mathcal S|K)$ memory when only two horizon layers are kept. PMF evaluation for a fixed policy costs $O(H|\mathcal S|\mathrm{BL})$ time and $O(|\mathcal S|L)$ memory, where $L$ is the number of reachable return-grid values. This is more expensive than scalar expected-return DP, whose backup stores one value per state, but it replaces exponential scenario-tree enumeration with dynamic programming over rank slices and return-grid distributions.

	The tabular sizes in our experiments range from 25-state GridWorld (GW) instances with horizon $H=20$ and four actions, to Betting Game (BG) with horizon $H=10$, 1011 time-augmented states, and six bet actions, and Inventory Control (IC) with horizon $H=10$, 4411 time-augmented states, and up to 21 order actions. Table~\ref{tab:main_param_sensitivity} summarizes representative sensitivity checks for these benchmark families.

	\begin{table}[t]
		\centering
		\setlength{\tabcolsep}{3pt}
		\caption{Representative discretization sensitivity. Gaps are percentages of the return span. $\Delta_K$ compares the smallest and largest tested $K_{\max}$; $\Delta_c$ is the gap variation at fixed rank budget. The GW row reports medians over GW1--GW4.}
		\label{tab:main_param_sensitivity}
		\begin{tabular}{@{}llcc@{}}
			\toprule
			Bench. & Risk setting & $\Delta_K$ & $\Delta_c$ \\
			\midrule
			BG & CVaR $\alpha=.20$ & $1.25\to0.078$ & $0.234$ \\
			BG & WOWA $\beta=2$   & $0.015\to0.006$ & $0.006$ \\
			GW & CVaR $\alpha=.10$ & $3.74\to0.11$ & $0.15$ \\
			IC & CVaR $\alpha=.02$ & $15.8\to1.78$ & $2.06$ \\
			\bottomrule
		\end{tabular}
	\end{table}

	The sensitivity results should be read as a compute-accuracy trade-off. The rank resolution $K$ is the main dial: increasing $K_{\max}$ consistently tightens the reported residual gap, with some runs crossing the strict tolerance and others stopping with an explicit remaining bound. In contrast, varying $c$ over $\{25,50,100,200\}$ changes the certificate gap only marginally in these runs, indicating that return-grid rounding is not the limiting factor once $c$ is moderately large.
	
	Our practical rule is therefore simple: set $c$ so that $H/(2c)$ is below the desired numerical tolerance, then increase $K$ by doubling until the certificate target is met or the remaining residual gap is acceptable. When the tested budget is insufficient, ERQDP reports the residual gap, allowing the remaining possible suboptimality due to approximation to be bounded rather than left unspecified.
	
	\section{Conclusion}

	We introduced ERQDP for finite-horizon tabular MDP planning under resolute root-level risk objectives. It preserves trajectory-level risk semantics while avoiding trajectory enumeration and sampling: it solves a rank--quantile surrogate by dynamic programming, evaluates explicit policies by deterministic PMF-DP up to a stated rounding budget, and reports a deterministic upper--lower optimality-gap bound. When this gap is below tolerance, ERQDP certifies the returned policy; otherwise it returns a residual bound under the chosen discretization budgets. Experiments on Betting Game, GridWorld, and Inventory Control illustrate the compute--accuracy tradeoff, support CVaR and WOWA in one pipeline, and show small gaps with fast risk-parameter sweeps on matched GridWorld comparisons. The sensitivity study confirms that rank resolution is the main certificate-tightening parameter, whereas PMF scale has little effect once return-grid rounding is negligible. Future work includes adaptive rank grids, sharper distortion-error bounds, and tighter surrogate-to-original certificates.

	\bibliographystyle{plainnat}
	\bibliography{references}

@article{HowardMatheson1972,
  author  = "Howard, Ronald A. and Matheson, James E.",
  title   = "Risk-sensitive {Markov} decision processes",
  journal = "Management Science",
  volume  = "18",
  number  = "7",
  pages   = "356--369",
  year    = "1972"
}

@article{Jaquette1976,
  author  = "Jaquette, Stratton C.",
  title   = "A Utility Criterion for {Markov} Decision Processes",
  journal = "Management Science",
  volume  = "23",
  number  = "1",
  pages   = "43--49",
  year    = "1976"
}

@book{Puterman1994,
  author    = "Puterman, Martin L.",
  title     = "{Markov} Decision Processes: Discrete Stochastic Dynamic Programming",
  publisher = "Wiley",
  year      = "1994"
}

@article{Strotz1955,
  author  = "Strotz, Robert H.",
  title   = "Myopia and Inconsistency in Dynamic Utility Maximization",
  journal = "Review of Economic Studies",
  volume  = "23",
  number  = "3",
  pages   = "165--180",
  year    = "1955"
}

@article{Ruszczynski2010,
  author  = "Ruszczy{\'n}ski, Andrzej",
  title   = "Risk-averse dynamic programming for {Markov} decision processes",
  journal = "Mathematical Programming",
  volume  = "125",
  number  = "2",
  pages   = "235--261",
  year    = "2010"
}

@article{BauerleJaskiewicz2024,
  author  = "B{\"a}uerle, Nicole and Ja{\'s}kiewicz, Anna",
  title   = "{Markov} decision processes with risk-sensitive criteria: an overview",
  journal = "Mathematical Methods of Operations Research",
  volume  = "99",
  pages   = "141--178",
  year    = "2024"
}

@article{BauerleOtt2011,
  author  = "B{\"a}uerle, Nicole and Ott, Jonathan",
  title   = "{Markov} decision processes with average-value-at-risk criteria",
  journal = "Mathematical Methods of Operations Research",
  volume  = "74",
  number  = "3",
  pages   = "361--379",
  year    = "2011"
}

@article{BauerleRieder2014,
  author  = "B{\"a}uerle, Nicole and Rieder, Ulrich",
  title   = "More risk-sensitive {Markov} decision processes",
  journal = "Mathematics of Operations Research",
  volume  = "39",
  number  = "1",
  pages   = "105--120",
  year    = "2014"
}

@article{Schmeidler1989,
  author  = "Schmeidler, David",
  title   = "Subjective Probability and Expected Utility without Additivity",
  journal = "Econometrica",
  volume  = "57",
  number  = "3",
  pages   = "571--587",
  year    = "1989"
}

@incollection{GonzalesP2020,
  title={Decision under uncertainty},
  author={Gonzales, Christophe and Perny, Patrice},
  booktitle={A Guided Tour of Artificial Intelligence Research: Volume I: Knowledge Representation, Reasoning and Learning},
  pages={549--586},
  year={2020},
  publisher={Springer}
}

@article{Wang1996,
  author  = "Wang, Shaun S.",
  title   = "Premium Calculation by Transforming the Layer Premium Density",
  journal = "ASTIN Bulletin",
  volume  = "26",
  number  = "1",
  pages   = "71--92",
  year    = "1996"
}

@article{Yager1988,
  author  = "Yager, Ronald R.",
  title   = "On Ordered Weighted Averaging Aggregation Operators in Multicriteria Decisionmaking",
  journal = "IEEE Transactions on Systems, Man, and Cybernetics",
  volume  = "18",
  number  = "1",
  pages   = "183--190",
  year    = "1988"
}

@article{Torra1997,
  author  = "Torra, Vicen{\c{c}}",
  title   = "The Weighted {OWA} Operator",
  journal = "International Journal of Intelligent Systems",
  volume  = "12",
  number  = "2",
  pages   = "153--166",
  year    = "1997"
}

@article{Artzner1999,
  author  = "Artzner, Philippe and Delbaen, Freddy and Eber, Jean-Marc and Heath, David",
  title   = "Coherent Measures of Risk",
  journal = "Mathematical Finance",
  volume  = "9",
  number  = "3",
  pages   = "203--228",
  year    = "1999"
}

@article{RockafellarUryasev2000,
  author  = "Rockafellar, R. Tyrrell and Uryasev, Stanislav",
  title   = "Optimization of Conditional Value-at-{Risk}",
  journal = "Journal of Risk",
  volume  = "2",
  number  = "3",
  pages   = "21--41",
  year    = "2000"
}

@article{AcerbiTasche2002,
  author  = "Acerbi, Carlo and Tasche, Dirk",
  title   = "On the Coherence of Expected Shortfall",
  journal = "Journal of Banking \& Finance",
  volume  = "26",
  number  = "7",
  pages   = "1487--1503",
  year    = "2002"
}

@article{Acerbi2002,
  author  = "Acerbi, Carlo",
  title   = "Spectral Measures of Risk: A Coherent Representation of Subjective Risk Aversion",
  journal = "Journal of Banking \& Finance",
  volume  = "26",
  number  = "7",
  pages   = "1505--1518",
  year    = "2002"
}

@inproceedings{Chow2015,
  author    = "Chow, Yinlam and Tamar, Aviv and Mannor, Shie and Pavone, Marco",
  title     = "Risk-Sensitive and Robust Decision-Making: A {CVaR} Optimization Approach",
  booktitle = "Advances in Neural Information Processing Systems 28 (NeurIPS)",
  year      = "2015"
}

@article{DingFeinberg2022,
  author  = "Ding, Rui and Feinberg, Eugene A.",
  title   = "{CVaR} Optimization for {MDPs}: Existence and Computation of Optimal Policies",
  journal = "SIGMETRICS Performance Evaluation Review",
  volume  = "50",
  number  = "2",
  pages   = "39--41",
  year    = "2022",
  month   = "aug"
}

@inproceedings{Lin2023a,
  author    = "Lin Hau, Jia and Petrik, Marek and Ghavamzadeh, Mohammad",
  title = "Entropic Risk Optimization in Discounted {MDPs}",
  booktitle = "Proceedings of The 26th International Conference on Artificial Intelligence and Statistics (AISTATS)",
  series    = "Proceedings of Machine Learning Research",
  volume    = "206",
  pages = "47--76",
  year      = "2023"
}

@inproceedings{Lin2023b,
  author    = "Lin Hau, Jia  and Delage, Erick and Ghavamzadeh, Mohammad and Petrik, Marek",
  title     = "On Dynamic Programming Decompositions of Static Risk Measures in Markov Decision Processes",
  booktitle = "Advances in Neural Information Processing Systems (NeurIPS)",
  year      = "2023",
  volume    = "36",
  pages     = "51734--51757",
  note      = "Also available as arXiv:2304.12477"
}

@article{AvilaPires2025,
	author  = "{\'A}vila Pires, Bernardo and Rowland, Mark and Borsa, Diana and Guo, Zhaohan Daniel and Khetarpal, Khimya and Barreto, Andr{\'e} and Abel, David and Munos, R{\'e}mi and Dabney, Will",
	title   = "Optimizing Return Distributions with Distributional Dynamic Programming",
	journal = "Journal of Machine Learning Research",
	volume  = "26",
	number  = "185",
	pages   = "1--90",
	year    = "2025"
}

@misc{Marthe2025,
  author       = "Marthe, Alexandre and Bounan, Samuel and Garivier, Aur{\'e}lien and Vernade, Claire",
  title        = "Efficient Risk-Sensitive Planning via Entropic Risk Measures",
  howpublished = "arXiv preprint",
  year         = "2025",
  eprint       = "2502.20423",
  archivePrefix= "arXiv",
  primaryClass = "cs.LG",
  note         = "Also circulated as HAL preprint hal-04967311v2"
}

@misc{MortensenTalebi2025,
  author       = "Mortensen, Oliver and Talebi, Mohammad Sadegh",
  title        = "Recursive Entropic Risk Optimization in Discounted {MDPs}: Sample Complexity Bounds with a Generative Model",
  howpublished = "arXiv preprint",
  year         = "2025",
  eprint       = "2506.00286",
  archivePrefix= "arXiv",
  primaryClass = "cs.LG",
}

@inproceedings{Su2025,
	author    = "Su, Xihong and Petrik, Marek and Grand-Cl{\'e}ment, Julien",
	title     = "Risk-Averse Total-Reward {MDPs} with {ERM} and {EVaR}",
	booktitle = "Proceedings of the AAAI Conference on Artificial Intelligence",
	volume    = "39",
	pages     = "20646--20654",
	year      = "2025",
	doi       = "10.1609/aaai.v39i19.34275"
}

@inproceedings{Zhang2024,
  author    = "Zhang, Dake and Lyu, Boxiang and Qiu, Shuang and Kolar, Mladen and Zhang, Tong",
  title     = "Pessimism Meets Risk: Risk-Sensitive Offline Reinforcement Learning",
  booktitle = "Proceedings of the 41st International Conference on Machine Learning (ICML)",
  series    = "Proceedings of Machine Learning Research",
  volume    = "235",
  pages     = "59459--59489",
  year      = "2024",
  publisher = "PMLR"
}

@inproceedings{YuShen2022,
	author    = "Yu, Xian and Shen, Siqian",
	title     = "Risk-Averse Reinforcement Learning via Dynamic Time-Consistent Risk Measures",
	booktitle = "Proceedings of the IEEE Conference on Decision and Control (CDC)",
	pages     = "2307--2312",
	year      = "2022",
	publisher = "IEEE",
	DOI       = "10.1109/CDC51059.2022.9992450"
}

@article{Yaari1987,
  author  = "Yaari, Menahem E.",
  title   = "The Dual Theory of Choice under Risk",
  journal = "Econometrica: Journal of the Econometric Society",
  volume  = "55",
  number  = "1",
  pages   = "95--115",
  publisher="JSTOR",
  year    = "1987"
}

@inproceedings{Kolobov2012,
  author    = "Kolobov, Andrey and Mausam and Weld, Daniel S.",
  title     = "A Theory of Goal-Oriented {MDPs} with Dead Ends",
  booktitle = "Proceedings of the Conference on Uncertainty in Artificial Intelligence (UAI)",
  pages     = "438--447",
  year      = "2012"
}

@inproceedings{Ahmadi2021,
	author    = "Ahmadi, Mohamadreza and Dixit, Anushri and Burdick, Joel W. and Ames, Aaron D.",
	title     = "Risk-Averse Stochastic Shortest Path Planning",
	booktitle = "Proceedings of the 60th IEEE Conference on Decision and Control (CDC)",
	pages     = "5199--5204",
	year      = "2021",
	publisher = "IEEE",
	doi       = "10.1109/CDC45484.2021.9683527"
}

@inproceedings{Rigter2022,
	author    = "Rigter, Marc and Duckworth, Paul and Lacerda, Bruno and Hawes, Nick",
	title     = "Planning for Risk-Aversion and Expected Value in {MDPs}",
	booktitle = "Proceedings of the International Conference on Automated Planning and Scheduling (ICAPS)",
	volume    = "32",
	pages     = "307--315",
	year      = "2022",
	doi       = "10.1609/icaps.v32i1.19814"
}

@misc{XiaPan2025,
  author       = "Xia, Li and Pan, Jinyan",
  title        = "{Markov} Decision Processes with Value-at-{Risk} Criterion",
  howpublished = "arXiv preprint",
  year         = "2025",
  eprint       = "2507.22355",
  archivePrefix= "arXiv"
}

@article{Allais1953,
	author  = "Allais, Maurice",
	title   = "Le comportement de l'homme rationnel devant le risque: critique des postulats et axiomes de l'{\'e}cole am{\'e}ricaine",
	journal = "Econometrica",
	volume  = "21",
	number  = "4",
	pages   = "503--546",
	year    = "1953"
}

@article{BenTalTeboulle1986,
	author  = "Ben-Tal, Aharon and Teboulle, Marc",
	title   = "Expected Utility, Penalty Functions, and Duality in Stochastic Nonlinear Programming",
	journal = "Management Science",
	volume  = "32",
	number  = "11",
	pages   = "1445--1466",
	year    = "1986",
	doi     = "10.1287/mnsc.32.11.1445"
}

@article{Tamar2017,
	author  = "Tamar, Aviv and Chow, Yinlam and Ghavamzadeh, Mohammad and Mannor, Shie",
	title   = "Sequential Decision Making with Coherent Risk",
	journal = "IEEE Transactions on Automatic Control",
	volume  = "62",
	number  = "7",
	pages   = "3323--3338",
	year    = "2017",
	doi     = "10.1109/TAC.2016.2644871"
}

@inproceedings{YuYing2023,
	author    = "Yu, Xian and Ying, Lei",
	title     = "On the Global Convergence of Risk-Averse Policy Gradient Methods with Expected Conditional Risk Measures",
	booktitle = "Proceedings of the 40th International Conference on Machine Learning",
	series    = "Proceedings of Machine Learning Research",
	volume    = "202",
	pages     = "40425--40451",
	year      = "2023",
	publisher = "PMLR"
}

@article{CoacheJaimungal2024,
	author  = "Coache, Anthony and Jaimungal, Sebastian",
	title   = "Reinforcement Learning with Dynamic Convex Risk Measures",
	journal = "Mathematical Finance",
	volume  = "34",
	number  = "2",
	pages   = "557--587",
	year    = "2024",
	doi     = "10.1111/mafi.12388",
	note    = "Published online 2023-04-17"
}

@article{JeantetSpanjaard2011,
	author  = "Jeantet, Gildas and Spanjaard, Olivier",
	title   = "Computing Rank Dependent Utility in Graphical Models for Sequential Decision Problems",
	journal = "Artificial Intelligence",
	volume  = "175",
	number  = "7--8",
	pages   = "1366--1389",
	year    = "2011",
	doi     = "10.1016/j.artint.2010.11.019"
}

@inproceedings{JeantetPernySpanjaard2012,
	author    = "Jeantet, Gildas and Perny, Patrice and Spanjaard, Olivier",
	title     = "Sequential Decision Making with Rank Dependent Utility: A Minimax Regret Approach",
	booktitle = "Proceedings of the AAAI Conference on Artificial Intelligence",
	volume    = "26",
	pages     = "1931--1937",
	year      = "2012",
	doi       = "10.1609/aaai.v26i1.8399"
}

@article{OgryczakPernyWeng2013,
	author  = {Ogryczak, W{\l}odzimierz and Perny, Patrice and Weng, Paul},
	title   = {A Compromise Programming Approach to Multiobjective {Markov} Decision Processes},
	journal = {International Journal of Information Technology \& Decision Making},
	volume  = {12},
	number  = {05},
	pages   = {1021--1053},
	year    = {2013}

}

@article{hong87,
	author  = "Chew, Soo Hong and Karni, Edi and Safra, Zvi",
	title   = "Risk Aversion in the Theory of Expected Utility with Rank Dependent Probabilities",
	journal = "Journal of Economic Theory",
	volume  = "42",
	number  = "2",
	pages   = "370--381",
	year    = "1987",
	doi     = "10.1016/0022-0531(87)90093-7"
}

@article{Dhaene2012,
  author  = {Dhaene, Jan and Kukush, Alexander and Linders, Dennis and Tang, Qihe},
  title   = {Remarks on quantiles and distortion risk measures},
  journal = {European Actuarial Journal},
  year    = {2012},
  volume  = {2},
  pages   = {319--328},
  doi     = {10.1007/s13385-012-0058-0}
}
	
	\newpage
	\onecolumn
	
	\title{Long-Term Sequential Decision Making Under Risk\\(Supplementary Material)}
	\maketitle
	
	\appendix
	\section{Toy slice example for quota filling}
	\label{app:toy_slice_example}
	
	Consider a lottery with outcomes $x_1\ge x_2\ge x_3\ge x_4\ge x_5$ and probabilities
	$m=(0.1,0.2,0.2,0.3,0.2)$, so $C=(0.1,0.3,0.5,0.8,1)$.
	Let $K=2$ with grid $p_0=0$, $p_1=0.5$, $p_2=1$.
	Then slice 1 averages the top half-mass: $\bar Q_X(1)=\frac{1}{0.5}\big(0.1x_1+0.2x_2+0.2x_3\big)$,
	and slice 2 averages the bottom half-mass: $\bar Q_X(2)=\frac{1}{0.5}\big(0.3x_4+0.2x_5\big)$.
	\label{app:extra}

	\section{GridWorld comparison details}
	\label{app:gridworld_audit}
	
	This section documents the precise aggregation procedure used to produce Table~\ref{tab:gridworld_cvar_compare}.
	
	\subsection{GridWorld layouts and settings}

	\label{app:gw_layouts}
	The four GridWorld instances used in the matched comparisons are denoted GW1--GW4. They use row-major cell numbering, start at cell $0$, and follow the slip rule stated in Section~\ref{sec:gridworld_compare}. Table~\ref{tab:gw_settings} lists the exact settings, and Figure~\ref{fig:gw_layouts} shows the corresponding reward layouts.
	
	\begin{table*}[t]
		\centering
		\caption{GridWorld settings used in the main comparison. All instances are $5\times5$ grids with start cell $0$, horizon $H=20$, and discount $\gamma=0.9$.}
		\label{tab:gw_settings}
		\begin{tabular}{lcccc}
			\toprule
			Instance & Slip $p_{\mathrm{slip}}$ & Goal cell/reward & Penalty cells/reward & Compared setting \\
			\midrule
			GW1 & $0.1$ & $21/+50$ & $10,12/-15$ & CVaR sweep \\
			GW2 & $0.3$ & $21/+50$ & $10,12/-15$ & CVaR sweep \\
			GW3 & $0.1$ & $21/+50$ & $12/-5$ & CVaR sweep \\
			GW4 & $0.1$ & $24/+50$ & $16/-30$ & CVaR sweep \\
			\bottomrule
		\end{tabular}
	\end{table*}
	
	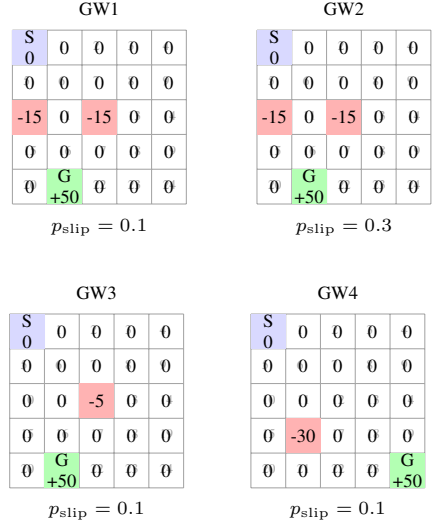
\begin{figure*}[t]
		\centering
		\newcommand{\gwgrid}[4]{
			\begin{tikzpicture}[scale=0.46, every node/.style={font=\scriptsize}]
				\draw[step=1,gray!70] (0,0) grid (5,5);
				
				\foreach \r in {0,...,4}{
					\foreach \c in {0,...,4}{
						\pgfmathtruncatemacro{\idx}{5*(4-\r)+\c}
						\node[anchor=north west, gray!70, font=\tiny] at (\c+0.05,\r+0.95) {\idx};
					}
				}
				
				\foreach \r in {0,...,4}{
					\foreach \c in {0,...,4}{
						\node at (\c+0.5,\r+0.5) {0};
					}
				}
				
				\fill[blue!15] (0,4) rectangle (1,5);
				\node at (0.5,4.5) {\shortstack{S\\0}};
				
				#2
				
				#3
				
				\node[anchor=south] at (2.5,5.15) {#1};
				\node[anchor=north] at (2.5,-0.15) {#4};
		\end{tikzpicture}}
		
		\gwgrid{GW1}
		{
			\fill[red!30] (0,2) rectangle (1,3); \node at (0.5,2.5) {-15};
			\fill[red!30] (2,2) rectangle (3,3); \node at (2.5,2.5) {-15};
		}
		{
			\fill[green!30] (1,0) rectangle (2,1); \node at (1.5,0.5) {\shortstack{G\\+50}};
		}
		{$p_{\mathrm{slip}}=0.1$}
		\hspace{0.7cm}
		\gwgrid{GW2}
		{
			\fill[red!30] (0,2) rectangle (1,3); \node at (0.5,2.5) {-15};
			\fill[red!30] (2,2) rectangle (3,3); \node at (2.5,2.5) {-15};
		}
		{
			\fill[green!30] (1,0) rectangle (2,1); \node at (1.5,0.5) {\shortstack{G\\+50}};
		}
		{$p_{\mathrm{slip}}=0.3$}
		
		\vspace{0.4cm}
		
		\gwgrid{GW3}
		{
			\fill[red!30] (2,2) rectangle (3,3); \node at (2.5,2.5) {-5};
		}
		{
			\fill[green!30] (1,0) rectangle (2,1); \node at (1.5,0.5) {\shortstack{G\\+50}};
		}
		{$p_{\mathrm{slip}}=0.1$}
		\hspace{0.7cm}
		\gwgrid{GW4}
		{
			\fill[red!30] (1,1) rectangle (2,2); \node at (1.5,1.5) {-30};
		}
		{
			\fill[green!30] (4,0) rectangle (5,1); \node at (4.5,0.5) {\shortstack{G\\+50}};
		}
		{$p_{\mathrm{slip}}=0.1$}
		
		\caption{GridWorld reward layouts for GW1--GW4. Ordinary cells have reward $0$; $S$ is the start cell and $G$ the goal cell. Small gray labels show row-major cell indices.}
		\label{fig:gw_layouts}
	\end{figure*}
	
	\textbf{(1) Certificate-backed planning and an explicit ``how close'' signal.}
	For all moderate risk levels and the risk-neutral setting ($\alpha\ge0.10$, including $\alpha=1$), every matched GW instance has an ERQDP gap below the prespecified tolerance. This matters because ERQDP reports either an $\varepsilon$-certificate or a residual optimality-gap bound, whereas the \cite{AvilaPires2025} results report empirical performance without a solver-side certificate.
	
	\textbf{(2) Deployment-style speed for risk sweeps.}
	For $\alpha\ge0.10$, the reported training-time/solve-time speedups range from $78\times$ to $928\times$, with ERQDP solve times dropping to single-digit seconds for $\alpha\ge0.25$. This is useful for risk sweeps because ERQDP directly re-solves each setting without a training loop.
	
	\textbf{(3) Tail-return behavior.}
	For $\alpha\ge0.10$, ERQDP has positive median $\Delta\mathrm{CVaR}_\alpha$ for every $\alpha\in{0.10,0.25,0.50,1.00}$, while also improving or largely preserving mean return. Thus, in the moderate-to-risk-neutral regime, the certificate-backed solutions are not obtained by simply sacrificing average performance; it often yields \emph{Pareto-improvements}
	(higher tail-return \emph{and} higher mean) at moderate-to-risk-neutral settings, where tail estimates are statistically stable.
	
	\paragraph{Extreme tails ($\alpha\le 0.05$): tradeoff between guarantees and computation.}
	At the most extreme tails ($\alpha=0.01,0.05$), ERQDP is less frequently certified (under prespecified  $varepsilon$ )and the baseline can show better estimated tail-return.
	This regime is intrinsically hard: the objective concentrates on a vanishing fraction of trajectories, amplifying discretization demands
	(and, for any empirical method, requiring large sample sizes to stabilize tail estimates).
	Importantly, ERQDP still provides a principled \emph{compute--guarantee dial}: increasing internal resolution tightens the envelope,
	and the remaining gap quantifies exactly how far we are from a certificate when we stop at a compute budget.

	\paragraph{Beyond CVaR: distortion-risk sweep with WOWA ($\beta\ge 1$).}
	While the \cite{AvilaPires2025} GridWorld results are reported for CVaR, ERQDP supports a broader family of distortions under the same certified DP framework.
	We include WOWA with $\beta\ge 1$ (risk-neutral at $\beta=1$; more tail-sensitive as $\beta$ increases),
	enabling risk-aversion sweeps without changing the algorithmic pipeline or losing the certificate mechanism.
	
	\paragraph{Risk-seeking comparison.}
	\cite{AvilaPires2025} proposes OCVaR as a risk-seeking variant.
	On the four matched GridWorlds (GW1--GW4), ERQDP-WOWA (e.g., $\beta{=}0.9$) matches OCVaR's optimistic upper-tail score at $\tau{=}0.1$
	(median $\Delta \mathrm{OCVaR}_{0.10}\approx 0$) while improving mean return (median $+3.03$) and the lower tail
	(median $+21.07$ in $\mathrm{CVaR}_{0.10}$), with fewer penalty visits, and runs $28$--$41\times$ faster (median $33.85\times$).In these matched instances, OCVaR preserves the optimistic-tail score but can reduce lower-tail return, whereas ERQDP-WOWA with $\beta<1$ attains comparable optimistic-tail values while maintaining stronger lower-tail metrics and solver-side optimality gaps/certificates.
 	
	\paragraph{Matched comparison.}
	Table~\ref{tab:gridworld_cvar_compare} aggregates only entries with the same GridWorld instance and risk parameter in both methods.
	
	\paragraph{Metrics and sign convention.}
	All comparisons are in return units, so higher is better. We use lower-tail $\mathrm{CVaR}_\alpha$ and report median differences
	$\Delta \mathrm{CVaR}_\alpha := \mathrm{CVaR}_\alpha^{\textsc{erqdp}}-\mathrm{CVaR}_\alpha^{\textsc{dm}}$
	and $\Delta \mathbb{E}[G] := \mathbb{E}[G]^{\textsc{erqdp}}-\mathbb{E}[G]^{\textsc{dm}}$.
	
	\paragraph{Speedup.}
	ERQDP time is end-to-end planning wall-clock time. Since the \citet{AvilaPires2025} baseline is training-based, we report speedup as
	\[
	\text{Speedup}=\frac{\text{reported baseline training time}}{\text{ERQDP solve time}},
	\]
	and take the median over matched instances for each $\alpha$. Thus the timing comparison is a training-time/solve-time diagnostic, not an identical-runtime-mode comparison.
	
	\paragraph{Optimality-gap interpretation.}
	ERQDP reports a final upper--lower optimality-gap bound for the returned policy. Table~\ref{tab:gridworld_cvar_compare} summarizes this quantity by its median over matched instances for each $\alpha$. At extreme tails ($\alpha\le0.05$), tighter bounds require higher rank resolution; otherwise ERQDP reports the remaining residual bound, whereas the baseline reports empirical performance without a solver-side optimality certificate.
	
	\paragraph{Moderate-to-risk-neutral regime.}
	For $\alpha\ge0.10$, Table~\ref{tab:gridworld_cvar_compare} shows small median optimality gaps, large training-time/solve-time speedups, and positive median $\Delta\mathrm{CVaR}_\alpha$, with typically positive $\Delta\mathbb{E}[G]$.

	\subsection{Risk-Seeking Objectives: OCVaR vs. ERQDP-WOWA}
	\label{app:ocvar_vs_wowa}
	
	\paragraph{Setup.}
	We compare OCVaR and ERQDP-WOWA on the four matched GridWorld instances GW1--GW4. Table~\ref{tab:ocvar_vs_wowa_summary} summarizes medians, and Table~\ref{tab:ocvar_vs_wowa_perenv} gives the per-instance rows.
	
	\paragraph{What is ``risk-seeking'' here (empirical, from the reported metrics).}
	Standard CVaR targets the \emph{lower} tail and is inherently risk-averse.
	\cite{AvilaPires2025}'s OCVaR is designed to act as an \emph{optimistic} (upper-tail) functional.
	This is directly visible in the reported metrics: on GW1/2/8, the reported $\mathrm{OCVaR}_{0.10}$ saturates at the environment's maximum return,
	while the lower-tail $\mathrm{CVaR}_{0.10}$ can become strongly negative under risk-seeking tuning.
	In contrast, ERQDP achieves risk-seeking behavior through WOWA with $\beta<1$ (optimistic distortion), while still producing an
	$\varepsilon$-certificate (or a final gap).

	\begin{table*}[t]
		\centering
		\setlength{\tabcolsep}{3pt}
		\caption{Risk-seeking summary over the four matched GridWorld environments (GW1--GW4).
			Values are medians over GW1--GW4. Higher return is better; lower penalty-visits is better. Speedup compares \cite{AvilaPires2025} reported training time with ERQDP solve time.}
		\label{tab:ocvar_vs_wowa_summary}
		\resizebox{\columnwidth}{!}{%
			\begin{tabular}{l r r r r r r r}
				\toprule
				Method & Param & Median $\mathbb{E}[G]$ $\uparrow$ & Median $\mathrm{OCVaR}_{0.10}$ $\uparrow$ & Median $\mathrm{CVaR}_{0.10}$ $\uparrow$ &
				Median Penalty $\downarrow$ & Time(s) $\downarrow$ \\
				\midrule
				OCVaR & $\tau=0.1$ & 23.13 & 32.81 & -4.27 & 0.238 & 1751 (train)   \\
				OCVaR & $\tau=0.9$ & 23.90 & 32.59 & 1.86  & 0.224 & 1751 (train)  \\
				\addlinespace
				ERQDP-WOWA & $\beta=0.9$ & 26.82 & 32.80 & 16.80 & 0.030 & 51.8 (solve)  \\
				ERQDP-WOWA & $\beta=0.3$ & 26.82 & 32.80 & 16.80 & 0.030 & 805 (solve)    \\
				\bottomrule
			\end{tabular}%
		}
	\end{table*}
	
	\begin{table*}[t]
		\centering
		\setlength{\tabcolsep}{3pt}
		\caption{Per-environment risk-seeking comparison: \cite{AvilaPires2025} OCVaR ($\tau=0.1$, seed-mean) vs.\ ERQDP-WOWA ($\beta=0.9$).
			Speedup is reported training time divided by ERQDP solve time; Gap is the ERQDP optimality-gap bound.}
		\label{tab:ocvar_vs_wowa_perenv}
		\begin{tabular}{l r r r r r r r r}
			\toprule
			Env & $\mathbb{E}[G]_{\textsc{dm}}$ & $\mathrm{CVaR}_{0.10,\textsc{dm}}$ & $\mathrm{OCVaR}_{0.10,\textsc{dm}}$ &
			$\mathbb{E}[G]_{\textsc{er}}$ & $\mathrm{CVaR}_{0.10,\textsc{er}}$ & $\mathrm{OCVaR}_{0.10,\textsc{er}}$ &
			Speedup($\times$) & Gap \\
			\midrule
			GW1 & 26.51 & -9.25  & 32.81 & 30.11 & 17.21 & 32.80 & 32.62 & 4.85e-3 \\
			GW2 & 17.90 & -24.17 & 32.81 & 23.53 & 4.59  & 32.80 & 27.95 & 4.41e-3 \\
			GW3 & 30.37 & 23.41  & 32.81 & 31.48 & 25.23 & 32.80 & 40.92 & 4.89e-3 \\
			GW4 & 19.75 & 0.71   & 23.91 & 22.20 & 16.39 & 23.91 & 35.07 & 4.18e-3 \\
			\bottomrule
		\end{tabular}
	\end{table*}
	
	\paragraph{Interpretation (what the results actually show).}
	On these four shared environments, OCVaR can indeed behave ``risk-seeking'' in the sense of maximizing an upper-tail score
	($\mathrm{OCVaR}_{0.10}$), but this comes with a measurable downside: the lower-tail return ($\mathrm{CVaR}_{0.10}$) degrades sharply
	(median becomes negative at $\tau=0.1$) and penalty exposure increases.
	ERQDP-WOWA ($\beta<1$) matches the same upper-tail saturation where it is achievable, while simultaneously maintaining
	a much stronger lower tail (median $\mathrm{CVaR}_{0.10}\approx 16.8$) and fewer penalty visits.

	\paragraph{Certified optimality and speed.}
	A key advantage of ERQDP is that it reports either an explicit $\varepsilon$-certificate for the returned policy or, when the strict prespecified tolerance is not reached, the final residual gap.
	For both $\beta=0.9$ and $\beta=0.3$, ERQDP certifies all four shared environments.
	The speedups compare \cite{AvilaPires2025}'s reported training runtime with ERQDP solve time on the same tasks; they should therefore be read as training-time/solve-time diagnostics, not as a same-mode planner-runtime comparison.
	(see Table~\ref{tab:ocvar_vs_wowa_perenv}) It can be seen while more extreme optimism ($\beta=0.3$) increases runtime due to tighter certification
	(needing larger internal resolution), yet still preserves the certificate mechanism.

	\section{Extended BG Results and ERQDP Certificates/Gaps}
	\label{app:bg_extended}
	Table~\ref{tab:bg_wowa_sweep_full} reports the WOWA sweep, and Table~\ref{tab:bg_erqdp_extended} gives the extended ERQDP certificate/gaps and rollout summaries.
	\begin{table*}[t]
		\centering
		\setlength{\tabcolsep}{3pt}
		\caption{BG: ERQDP-WOWA sweep (cost; lower is better). We report both the optimized WOWA objective value (not comparable to CVaR)
			and common cost-risk metrics ($\mathrm{VaR}$/$\mathrm{CVaR}$) for cross-method risk-handling comparison.}
		\label{tab:bg_wowa_sweep_full}
		
		\begin{tabular}{c r r r r r r l r r}
			\toprule
			$\beta$ & Mean & $\mathrm{CVaR}_{0.2}$ & $\mathrm{VaR}_{0.2}$ & $\mathrm{CVaR}_{0.02}$ & $\mathrm{VaR}_{0.02}$ & WOWA obj. & Gap & $K$ & Time(s) \\
			\midrule
			0.1 & 61.86 & 99.97 & 97 & 100 & 100 & 11.41 &   7.3e-02 & 17000 & 1621.1 \\
			0.5 & 61.73 & 99.71 & 94 & 100 & 100 & 42.13 &   9.4e-03 & 17000 & 1586.6 \\
			1 & 61.65 & 97.91 & 87 & 100 & 100 & 61.92 &    0.0e+00 & 256 & 82.4 \\
			2 & 62.03 & 96.59 & 88 & 100 & 100 & 78.66 &    0.0e+00 & 1024 & 141.7 \\
			3 & 95 & 95 & 95 & 95 & 95 & 95 &   4.4e-05 & 4096 & 428.6 \\
			5 & 95 & 95 & 95 & 95 & 95 & 95 &   1.5e-04 & 4096 & 364.7 \\
			\bottomrule
		\end{tabular}%
		
	\end{table*}

	The extended table includes 95\% confidence intervals over 20,000 rollouts, tail metrics, the final certification/optimality gap, and the rank-grid size $K$.
	
	\begin{table*}[t]
		\centering
		\setlength{\tabcolsep}{3pt}
		\caption{BG extended results (ERQDP). Mean is reported with 95\% CI in brackets. VaR/CVaR are computed on the cost distribution
			(i.e., the upper tail of cost, equivalent to the lower tail of return).}
		\label{tab:bg_erqdp_extended}
		\begin{tabular}{l l l r r r r l r r}
			\toprule
			Risk & Param & Mean [CI] & VaR$_{0.02}$ & CVaR$_{0.02}$ & VaR$_{0.2}$ & CVaR$_{0.2}$ & Gap & $K$ & Time(s) \\
			\midrule
			CVaR & $\alpha=0.02$ & 95.00 [95.00,95.00] & 95.00 & 95.00 & 95.00 & 95.00    & 2.8e-14 & 17000 & 1596.36 \\
			CVaR & $\alpha=0.2$  & 95.00 [95.00,95.00] & 95.00 & 95.00 & 95.00 & 95.00  & 1.4e-14 & 17000 & 1447.72 \\
			CVaR & $\alpha=1$    & 61.65 [61.24,62.07] & 100.00 & 100.00 & 87.00 & 97.91   & 0.0e+00 & 256 & 60.47 \\
			\addlinespace
			WOWA & $\beta=0.1$ & 61.86 [61.43,62.30] & 100.00 & 100.00 & 97.00 & 99.98 &     7.3e-02 & 17000 & 1621.07 \\
			WOWA & $\beta=0.5$ & 61.73 [61.30,62.16] & 100.00 & 100.00 & 94.00 & 99.71   & 9.4e-03 & 17000 & 1586.61 \\
			WOWA & $\beta=1$   & 61.65 [61.24,62.07] & 100.00 & 100.00 & 87.00 & 97.91   & 0.0e+00 & 256 & 82.41 \\
			WOWA & $\beta=2$   & 62.03 [61.62,62.44] & 100.00 & 100.00 & 88.00 & 96.59   & 0.0e+00 & 1024 & 141.69 \\
			WOWA & $\beta=3$   & 95.00 [95.00,95.00] & 95.00 & 95.00 & 95.00 & 95.00  & 4.4e-05 & 4096 & 428.59 \\
			WOWA & $\beta=5$   & 95.00 [95.00,95.00] & 95.00 & 95.00 & 95.00 & 95.00   & 1.5e-04 & 4096 & 364.66 \\
			\bottomrule
		\end{tabular}
	\end{table*}
	
	\paragraph{Takeaway.}
	Lower $\alpha$ (CVaR) or higher $\beta$ (WOWA) pushes toward policies that reduce catastrophic outcomes, which may increase average cost.
	Conversely, $\beta<1$ is risk-seeking: it emphasizes best-case outcomes, yielding optimistic distorted values even when average cost stays similar.

	\section{Inventory Control (IC): Detailed Comparison to Rigter et al.\ (ICAPS'22)}
	\label{app:ic_rigter}
	
	\paragraph{Setup and metric.}
	Following Rigter et al., IC is evaluated over 10 stages with $N{=}20$ and the episode cost defined as
	$C = 400 - \sum_{t=0}^{H-1}\mathrm{profit}_t$ (thus $C<400$ indicates net profit).
	We report CVaR$_{0.02}$ (mean cost of the worst $2\%$ runs) and mean cost.
	
	\paragraph{Results.}
	Table~\ref{tab:ic_rigter} shows that ERQDP improves the CVaR$_{0.02}$ objective substantially, corresponding
	to roughly doubling the implied worst-tail profit (since tail-profit $=400-\text{CVaR}_{0.02}$).
	
	\begin{table*}[t]
		\centering
		\caption{Inventory Control (IC). Lower cost is better. Tail-profit is $400-\text{CVaR}_{0.02}$.}
		\label{tab:ic_rigter}
		\begin{tabular}{lcccc}
			\toprule
			Method & CVaR$_{0.02}$ cost $\downarrow$ & Mean cost $\downarrow$ & Tail-profit $\uparrow$ & Gap bound? \\
			\midrule
			Rigter EV                      & 416.42 & 235.62 & $-16.42$ & -- \\
			Rigter CVaR-WC ($\alpha{=}0.02$) & 386.49 & 286.18 & $13.51$  & -- \\
			Rigter CVaR-EV ($\alpha{=}0.02$) & 386.92 & 250.38 & $13.08$  & -- \\
			\midrule
			\textbf{ERQDP-CVaR ($\alpha{=}0.02$)} & \textbf{373.26} & 250.29 & \textbf{26.74} & \checkmark \\
			\bottomrule
		\end{tabular}
	\end{table*}
	
	\paragraph{A~posteriori optimality-gap bound (certificate).}
	Beyond the empirical scores, ERQDP logs explicit bounds for the optimized risk objective:
	an upper bound $UB_{\mathrm{opt}}$, a certified lower bound $LB_{\mathrm{opt}}^{\mathrm{cert}}$,
	and the achieved policy value $V^\pi(s_0)$, inducing a computable policy-gap bound
	\[
	\mathrm{policy\_gap} = UB_{\mathrm{opt}} - V^\pi(s_0).
	\]
	For IC at $\alpha{=}0.02$ (with $K_{\max}{=}512$), ERQDP reports
	$UB_{\mathrm{opt}}=-287.747$, $LB_{\mathrm{opt}}^{\mathrm{cert}}=-458.747$, $V^\pi(s_0)=-373.264$,
	yielding \emph{policy gap = 85.52} ; i.e., an explicit bound on suboptimality for the returned policy.
	Rigter et al.\ do not provide an analogous bound in their evaluation protocol.

	\section{Discretization Sensitivity}
	\label{app:param_sensitivity}
	
	This section reports the sensitivity checks summarized in Table~\ref{tab:main_param_sensitivity}. The experiments use the same ERQDP pipeline as the main results and vary only discretization budgets. Table~\ref{tab:k_sensitivity_appendix} varies the rank budget $K_{\max}$, while Table~\ref{tab:c_sensitivity_appendix} varies the PMF grid scale $c$. BG denotes Betting Game, GW denotes GridWorld, and IC denotes Inventory Control; a star denotes certification at the target tolerance. The qualitative conclusion is that $K$ controls the certificate gap, while $c$ is stable once the return grid is fine enough.
	
	\begin{table}[!htbp]
		\centering
		\caption{Representative sensitivity to the rank resolution $K$. Larger $K$ tightens the rank approximation and typically decreases the certificate gap, at the cost of additional runtime. Gap/span is reported as a percentage of the return span; $\star$ marks rows whose gap is below the target tolerance (prespecified $\varepsilon$).}
		
		\label{tab:k_sensitivity_appendix}
		\begin{tabular}{lllrrrr}
			\toprule
			Environment & Risk & $K_{\max}$ & final $K$ & gap/span & time (s) \\
			\midrule
			Betting Game & CVaR $\alpha=0.2$ & 64 & 64 & 1.250\% & 9.3 \\
			Betting Game & CVaR $\alpha=0.2$ & 128 & 128 & 0.938\% & 18.2 \\
			Betting Game & CVaR $\alpha=0.2$ & 256 & 256 & 0.312\% & 28.0 \\
			Betting Game & CVaR $\alpha=0.2$ & 512 & 512 & 0.234\% & 48.5 \\
			Betting Game & CVaR $\alpha=0.2$ & 1024 & 1024 & 0.078\%$^{\star}$ & 131.0 \\
			Betting Game & WOWA $\beta=2$ & 64 & 64 & 0.015\%$^{\star}$ & 19.3 \\
			Betting Game & WOWA $\beta=2$ & 128 & 128 & 0.006\%$^{\star}$ & 26.6 \\
			Betting Game & WOWA $\beta=2$ & 256 & 128 & 0.006\%$^{\star}$ & 35.3 \\
			Betting Game & WOWA $\beta=2$ & 512 & 128 & 0.006\%$^{\star}$ & 32.0 \\
			Betting Game & WOWA $\beta=2$ & 1024 & 128 & 0.006\%$^{\star}$ & 28.7 \\
			GridWorld (median over GW1-4) & CVaR $\alpha=0.1$ & 64 & 64 & 3.736\% & 2.8 \\
			GridWorld (median over GW1-4) & CVaR $\alpha=0.1$ & 128 & 128 & 1.223\% & 4.7 \\
			GridWorld (median over GW1-4) & CVaR $\alpha=0.1$ & 256 & 256 & 0.741\% & 7.0 \\
			GridWorld (median over GW1-4) & CVaR $\alpha=0.1$ & 512 & 512 & 0.152\% & 12.0 \\
			GridWorld (median over GW1-4) & CVaR $\alpha=0.1$ & 1024 & 512--1024 & 0.112\% & 18.9 \\
			Inventory Control & CVaR $\alpha=0.02$ & 64 & 64 & 15.758\% & 839.6 \\
			Inventory Control & CVaR $\alpha=0.02$ & 128 & 128 & 9.628\% & 1240.0 \\
			Inventory Control & CVaR $\alpha=0.02$ & 256 & 256 & 2.063\% & 1960.0 \\
			Inventory Control & CVaR $\alpha=0.02$ & 512 & 512 & 1.782\% & 3260.0 \\
			\bottomrule
		\end{tabular}
	\end{table}
	
	\begin{table}[!htbp]
		\centering
		\caption{Representative sensitivity to the PMF return-grid scale $c$ at fixed rank budget. In these runs, changing $c$ over $\{25,50,100,200\}$ has little effect on the certificate gap, indicating that the rank approximation, rather than return rounding, is the dominant limiting factor.}
		\label{tab:c_sensitivity_appendix}
		\begin{tabular}{lllrrrr}
			\toprule
			Environment. & Risk & $c$ & final $K$ & gap/span & time (s) \\
			\midrule
			Betting Game & CVaR $\alpha=0.2$ & 25 & 512 & 0.234\% & 48.9 \\
			Betting Game & CVaR $\alpha=0.2$ & 50 & 512 & 0.234\% & 46.2 \\
			Betting Game & CVaR $\alpha=0.2$ & 100 & 512 & 0.234\% & 48.7 \\
			Betting Game & CVaR $\alpha=0.2$ & 200 & 512 & 0.234\% & 51.4 \\
			Betting Game & WOWA $\beta=2$ & 25 & 128 & 0.006\%$^{\star}$ & 27.0 \\
			Betting Game & WOWA $\beta=2$ & 50 & 128 & 0.006\%$^{\star}$ & 25.7 \\
			Betting Game & WOWA $\beta=2$ & 100 & 128 & 0.006\%$^{\star}$ & 31.4 \\
			Betting Game & WOWA $\beta=2$ & 200 & 128 & 0.006\%$^{\star}$ & 29.7 \\
			GridWorld (median over 4) & CVaR $\alpha=0.1$ & 25 & 512 & 0.152\% & 10.7 \\
			GridWorld (median over 4) & CVaR $\alpha=0.1$ & 50 & 512 & 0.152\% & 9.0 \\
			GridWorld (median over 4) & CVaR $\alpha=0.1$ & 100 & 512 & 0.152\% & 10.2 \\
			GridWorld (median over 4) & CVaR $\alpha=0.1$ & 200 & 512 & 0.152\% & 11.0 \\
			Inventory Control & CVaR $\alpha=0.02$ & 25 & 256 & 2.063\% & 1830.0 \\
			Inventory Control & CVaR $\alpha=0.02$ & 50 & 256 & 2.063\% & 1790.0 \\
			Inventory Control & CVaR $\alpha=0.02$ & 100 & 256 & 2.063\% & 1940.0 \\
			Inventory Control & CVaR $\alpha=0.02$ & 200 & 256 & 2.063\% & 2030.0 \\
			\bottomrule
		\end{tabular}
	\end{table}

	\FloatBarrier
	\section{Proofs}
	\label{app:proofs}
	
	\subsection{Proof of Lemma~\ref{lem:quota}}
	\begin{proof}
		Let the sorted outcomes be $x_1\ge\cdots\ge x_N$ with masses $m_i$ and cumulative masses $C_i=\sum_{j=1}^i m_j$. The best-to-worst quantile function is constant and equal to $x_i$ on the interval $(C_{i-1},C_i]$. For a grid slice $(p_{k-1},p_k]$, the slice-average quantile is therefore
		\[
		\bar Q_X(k)=\frac{1}{q_k}\sum_{i=1}^N x_i\,\lambda\big((C_{i-1},C_i]\cap(p_{k-1},p_k]\big).
		\]
		The quota-filling procedure scans the same outcome intervals in decreasing value order and allocates exactly the intersection length above to slice $k$. Its accumulator for slice $k$ is consequently the numerator of this display, and division by $q_k$ gives exactly $\bar Q_X(k)$. This holds for every slice, so the returned vector is the slice-average quantile vector.
	\end{proof}
	
	\subsection{Proof of Proposition~\ref{prop:rqdp-opt}}
	\begin{proof}
		We prove the claim by induction on the remaining horizon $h$. For $h=0$, all policies induce the zero return and the algorithm stores $V_0(s)=\mathbf 0$, which is optimal. Suppose the claim holds for $h-1$. For a fixed state $s$ and action $a$, the successor summaries $V_{h-1}(s')$ are, by the induction hypothesis, the optimal surrogate summaries for the continuation problems. The backup forms the discrete mixture over discounted one-step rewards plus successor slice values. By Lemma~\ref{lem:quota}, the sort-and-quota-fill operation computes the exact $K$-slice summary of this induced surrogate lottery. Since the surrogate objective is linear in the slice summary, namely $\langle\Delta\phi^{(K)},\bar Q\rangle$, selecting an action that maximizes this score gives an optimal $h$-step surrogate action at $s$. Applying this argument for all states and all horizons proves that Algorithm~\ref{alg:rqdp} returns a deterministic Markov policy maximizing the induced $K$-slice surrogate objective.
	\end{proof}
	
	\subsection{Proof of Proposition~\ref{prop:pmf-eval}}
	\begin{proof}
		For a fixed policy, the return distribution satisfies the usual law of total probability over the first transition. The recursion in Eq.~\eqref{eq:pmf-recursion} applies exactly this decomposition: for each successor $s'$, it shifts the already-computed PMF of the remaining return by the rounded one-step discounted reward, weights it by $P(s'\mid s,\pi_h(s))$, and sums over successors. Induction on $h$ proves that the recursion computes the PMF of the rounded return $\tilde G^\pi(s_0)$.
		
		If every discounted one-step reward is aligned with the grid, the shift is exact at every step, so $\tilde G^\pi(s_0)=G^\pi(s_0)$. Otherwise each rounded one-step reward differs from the true discounted reward by at most $1/(2c)$. Along any trajectory, summing over $H$ stages gives $|G^\pi(s_0)-\tilde G^\pi(s_0)|\le H/(2c)$ almost surely. If two random variables differ almost surely by at most $\eta$, then their quantile functions differ uniformly by at most $\eta$. A distortion risk functional is a monotone weighted integral of quantiles with total weight one, hence its value changes by at most $\eta$. Taking $\eta=H/(2c)$ gives the stated bound.
	\end{proof}
	
	\subsection{Proof of Theorem~\ref{thm:lipschitz}}
	\begin{proof}
		Let $Q_X$ denote the best-to-worst quantile function of $X$. The distortion value can be written as a Stieltjes integral $\rho_\phi(X)=\int_0^1 Q_X(p)\,d\phi(p)$. Therefore
		\[
		\rho_\phi(X)-\rho_\psi(X)=\int_0^1 Q_X(p)\,d(\phi-\psi)(p).
		\]
		Because $X\in[m,M]$ almost surely, $Q_X$ is monotone with total variation at most $M-m$. Since $\phi(0)=\psi(0)$ and $\phi(1)=\psi(1)$ for distortions, integration by parts gives
		\[
		\begin{aligned}
			\left|\int_0^1 Q_X\,d(\phi-\psi)\right|
			&=\left|\int_0^1 (\phi-\psi)\,dQ_X\right|\\
			&\le \|\phi-\psi\|_\infty (M-m),
		\end{aligned}
		\]
		which proves the claim.
	\end{proof}

\end{document}